\documentclass[english,11pt]{article}

\usepackage[margin=1in]{geometry}
\usepackage{bm}
\usepackage[utf8]{inputenc}
\usepackage{babel}
\usepackage{slantsc}
\usepackage{array}
\usepackage{ulem}
\usepackage{amsmath}
\usepackage{amsthm}
\usepackage{amssymb}
\usepackage{graphicx}
\usepackage{enumerate}
\usepackage{mathtools}
\usepackage{natbib}
\usepackage{url}
\usepackage{hyperref}
\usepackage{bbm}
\usepackage{tikz}
\usepackage{subfigure}
\usepackage{fix-cm}
\usepackage{multirow}
\usepackage{utopia}
\usepackage{mathrsfs}

\def\E{\mathbb{E}}
\def\P{\mathbb{P}}

\def\N{\mathbb{N}}
\def\S{\mathbb{S}}

\def\I{\mathcal{I}}
\def\R{\mathcal{R}}
\def\T{\mathcal{T}}
\def\D{\mathcal{D}}
\def\A{\mathcal{A}}
\def\K{\mathcal{K}}

\usepackage{pgfplots}
\usetikzlibrary{arrows,decorations.pathmorphing,backgrounds,fit,positioning,shapes.symbols,chains, decorations.pathreplacing,patterns,calc}

\definecolor{dark_purple}{rgb}{0.4, 0.0, 0.4}

\newtheorem{theorem}{Theorem}

\newtheorem{lemma}{Lemma}
\newtheorem*{theorem*}{Theorem}
\theoremstyle{definition}

\definecolor{PennBlue}{RGB}{001,031,091}
\definecolor{PennRed}{RGB}{153,0,0}
\usepackage{hyperref}
\hypersetup{
    pdfborder = {0 0 0},
    colorlinks,
    citecolor=PennRed,
    filecolor=PennRed,
    linkcolor=PennRed,
    urlcolor=blue
}

\newcommand{\real}{\mathbb{R}}

\newcommand{\II}[1]{\mathbb{I}_{\left\{#1\right\}}}
\newcommand{\PP}[1]{\mathbb{P}\left[#1\right]}
\newcommand{\EE}[1]{\mathbb{E}\left[#1\right]}

\newcommand{\ev}[1]{\left\{#1\right\}}
\newcommand{\pa}[1]{\left(#1\right)}

\newcommand{\wh}{\widehat}
\newcommand{\wt}{\widetilde}

\usepackage{todonotes}
\definecolor{PalePurp}{rgb}{0.66,0.57,0.66}

\newcommand{\tc}{\tilde{c}}
\newcommand{\tI}{\tilde{I}}
\newcommand{\EEt}[1]{\mathbb{E}_t\left[#1\right]}

\usepackage[blocks]{authblk}
\begin{document}

\title{On the Hardness of Inventory Management with Censored Demand Data}
\author{\fontsize{13}{13}\selectfont {G\'{a}bor Lugosi}\thanks{\fontsize{9}{9}\selectfont ICREA \& Department of Economics and Business, 
Universitat Pompeu Fabra \& Barcelona Graduate School of Economics (gabor.lugosi@gmail.com).}
\hspace{0.7in} {Mihalis G. Markakis}\thanks{\fontsize{9}{9}\selectfont Department of Economics and Business, Universitat Pompeu Fabra \& Barcelona School of Management (mihalis.markakis@upf.edu).}
\hspace{0.7in}{Gergely Neu}\thanks{\fontsize{9}{9}\selectfont Department of Information and Communication Technologies, Universitat Pompeu Fabra (gergely.neu@gmail.com).}}

\date{}
\maketitle

\begin{abstract}

We consider a repeated newsvendor problem where the inventory manager has no prior information about the demand, and can access only 
censored/sales data. In analogy to multi-armed bandit problems, the manager needs to simultaneously ``explore'' and ``exploit'' with her 
inventory decisions, in order to minimize the cumulative cost. We make no probabilistic assumptions---importantly, independence or time 
stationarity---regarding the mechanism that creates the demand sequence. Our goal is to shed light on the hardness of the problem, and to 
develop policies that perform well with respect to the regret criterion, that is, the difference between the cumulative cost of a policy and 
that of the best fixed action/static inventory decision in hindsight, uniformly over all feasible demand sequences. We show that a simple 
randomized policy, termed the Exponentially Weighted Forecaster, combined with a carefully designed cost estimator, achieves optimal scaling 
of the expected regret (up to logarithmic factors) with respect to all three key primitives: the number of time periods, the number of 
inventory decisions available, and the demand support. Through this result, we derive an important insight: the benefit from ``information 
stalking'' as well as the cost of censoring are both negligible in this dynamic learning problem, at least with respect to the regret 
criterion. Furthermore, we modify the proposed policy in order to perform well in terms of the tracking regret, that is, using as benchmark 
the best sequence of inventory decisions that switches a limited number of times. Numerical experiments suggest that the proposed approach 
outperforms existing ones (that are tailored to, or facilitated by, time stationarity) on nonstationary demand models. Finally, we consider 
the ``combinatorial'' version of the repeated newsvendor problem, that is, single-warehouse multi-retailer inventory management of a 
perishable product. We extend the proposed approach so that, again, it achieves near-optimal performance in terms of the regret.

\vspace{0.2in}

\noindent\textbf{Keywords:} repeated newsvendor problem, demand learning, censored observations, regret minimization, Exponentially Weighted Forecaster.
\newline

\end{abstract}

\baselineskip0.6cm

\setcounter{page} {1}


\section{Introduction} \label{Intr}

We consider the multi-period inventory management problem of a perishable product, like newspapers, fresh food, or certain pharmaceutical 
products, where no prior information is available about the demand for the product over the different periods. This may be the case when a 
new product is introduced to the market, or when the market conditions for an existing product change drastically, e.g., due to a major competitor 
entering/exiting the market, due to a major economic downturn etc. An additional complication comes from the fact that, often times, the 
inventory manager cannot observe the actual demand for the product due to \textit{censoring}: the firm is unable to measure or estimate 
accurately lost sales, so the manager has access only to sales data. However, the sales depend on the manager's prior inventory decisions, 
making inferences about the underlying demand much harder. In such scenarios, the inventory manager is faced with a dynamic learning 
problem, having to simultaneously ``explore'' with her inventory decisions in order to learn the underlying demand, as well as ``exploit,'' that is, focus mostly on decisions that are likely 
to incur low cost. Due to its practical importance and intellectual challenge, the problem of inventory management with demand learning through censored data has attracted significant attention from the academic community, leading to valuable insights as we detail below.

Our main motivation, and point of departure from the existing literature, stems from the fact that the demand for a product may very well be nonstationary: 
trends and seasonalities are very common in a demand time series; competition in the market that a firm operates may change over time, in terms of both 
the assortments and the prices offered; consumers may time their decisions strategically. Our goal is to develop a framework that incorporates, in a tractable 
way, the potentially nonstationary nature of the demand, to explore the fundamental limits of performance in this setting and propose suitable inventory 
management policies, and to shed light on the performance loss compared to the case where the demand is time-stationary. Accordingly, we adopt a ``nonstochastic'' 
viewpoint: we formulate the problem of inventory management under censored demand as a repeated game between the inventory manager and the 
market, without making any probabilistic assumptions on the mechanism via which the market generates the demand. We evaluate the performance 
of different policies with respect to the regret criterion, that is, the difference between the cumulative cost of a policy and the cumulative 
cost of the best fixed action/inventory decision in hindsight, for a given demand sequence, and provide performance guarantees that hold 
uniformly over all demand sequences.

The above viewpoint is also referred to in the literature as the ``adversarial'' approach, although this term is somewhat 
unfortunate: in our setting, the market chooses a sequence of demands for the different time periods arbitrarily, as far as the 
inventory manager is concerned. Importantly though, the market does not adapt its strategy according to the actions of the inventory manager. 
Indeed, it seems far-fetched to assume that an entire market adapts, and acts adversarially, to the inventory decisions of a firm. 
So, while our modeling framework can capture demand correlations and nonstationarities to a significant extent, it is certainly not a 
game-theoretic model.

In the remainder of the introduction we provide a detailed account of the existing literature on the repeated newsvendor problem with 
demand learning, we give some background on the nonstochastic approach to dynamic learning problems, and we highlight the main contributions of 
the present paper.


\bigskip\subsection{Related Literature} \label{LitRev}

Stochastic inventory theory is a field with a long history and rich literature. A detailed account of this literature is beyond the scope 
of our paper, so we refer the interested reader to \cite{P02}. It could be argued that one of the most influential contributions in this 
field is the development and analysis of the newsvendor model, where a manager has to make an inventory decision in anticipation of 
uncertain demand over a single selling period, aiming at minimizing the total expected overage and underage cost. Elementary arguments can 
show that the optimal inventory decision is a \textit{critical quantile} of the demand distribution. Equally important is the mathematical 
model of multi-period inventory management of a nonperishable product with uncertain demand, whose dynamic programming formulation gives 
rise to the optimality of $(s,S)$ inventory replenishment policies. To some extent both models are related to our work, as we consider a 
multi-period inventory management problem of a perishable product, that is, a repeated newsvendor problem. However, a common (and critical) 
assumption in both models, and in ``classical'' stochastic inventory theory overall, is that even though the inventory manager does not know 
the realization of future demand, she does have access to an accurate probabilistic description of it, for example, via historical data. We do not make this assumption in our work.

Throughout the years, there have been several attempts to relax the assumption that the correct distribution of future demand is available. 
The most followed approach is to assume that the demand belongs to a particular family of probability distributions, but one or more 
parameters are unknown. This \textit{parametric approach} is usually cast in a Bayesian learning framework: a prior on the parameters is 
also assumed, and the belief about the true parameter values is updated with observed demand samples through Bayes rule. Early works in that 
direction include \cite{S59}, \cite{K60}, and \cite{I64}, which focus on exponential families of demand distributions. \cite{MS66} and 
\cite{A85} consider variants of the problem and incorporate Bayesian learning into a dynamic programming framework, while \cite{CF71} uses 
the Kalman filtering approach to achieve efficient learning/forecasting. Finally, \cite{L90} shows the near-optimality of simple myopic 
inventory management policies, when combined with adaptive tuning of the parameters via Bayes rule or exponential smoothing.

A common characteristic of the above works is that the inventory manager has access to the realized demand, in order to update her beliefs. 
During stock-outs, however, it is often the case that excess demand is lost, making it very hard to measure or estimate the realized demand. 
In other words, on many occasions it may be more realistic to assume that the inventory manager has access only to the sales, that is, 
\textit{censored demand} data. The main insight here is that a dynamic analysis is required even in the case of a perishable product, and 
that the optimal inventory decision is higher than that of a Bayesian myopic policy, a phenomenon that is referred to in the literature as 
``information stalking.'' The intuition behind it is that this additional inventory gives, occasionally, some extra uncensored demand 
samples, which contribute towards learning the true parameter values and are, thus, useful in the future. Consequently, some level of 
``experimentation/exploration'' is necessary when dealing with censored demand. This result was first proved in \cite{HLW82} in the context 
of a perfectly competitive firm making output decisions in the presence of demand uncertainty, and later cast in an inventory management 
setting and strengthened in \cite{DPB02} and \cite{LSZ08}.\footnote{Interestingly, a similar insight has been derived in a dynamic pricing 
context in \cite{BO94}.} \cite{LP99} derives a closed-form expression for the Bayesian optimal inventory level if the demand belongs to the class of ``newsvendor 
distributions'' developed in \cite{BF91}, and confirms that it is optimal to enhance learning through stocking higher. Recently, \cite{BCMpre}, 
building on the framework of \cite{LP99}, provides both analytical and numerical evidence to the fact that, while there is cost in being 
myopic (instead of far-sighted, in the Dynamic Programming sense), this cost is actually quite small. Hence, a Bayesian myopic policy is 
near-optimal, apart from being easy to implement. Moreover, the cost of censoring, despite being not too large either, is about an order of 
magnitude greater than the cost of myopia, so the inventory manager should direct her efforts in measuring/estimating lost sales. 

The picture becomes more complicated in the case of nonperishable products: the inventory carried over from previous periods may force the inventory manager to stock higher or lower compared to the Bayesian myopic benchmark; see \cite{CP08}, where the effect of substitutable products is also studied. In a follow-up work, \cite{C10} develops improved bounds and heuristics for the problem.

A fundamental limitation, and the standard criticism against the parametric approach, is that if the parametric family adopted is not 
broad/flexible enough to capture the underlying demand process, estimating the best parameter values is of little help in really learning 
the demand, and then managing the inventory in a cost-effective way. Hence, in parallel to the aforementioned parametric approach to 
stochastic inventory theory with uncertainty regarding the demand distribution, a literature following a \textit{nonparametric approach} has 
also been developing. The setting here is one where the inventory manager has no prior information on the demand, other than the fact that 
it is mutually independent (and in some cases, also identically distributed) over different time periods, and she has access to censored 
demand data. \cite{BS00} and \cite{KT08} propose stochastic approximation algorithms for ordering and pricing, and prove their asymptotic 
optimality, without providing any rates of convergence though.  Adaptive value-estimation methods that take advantage of the convexity of 
the newsvendor cost function are presented in \cite{GP01} and \cite{PRT04}. Again, the convergence of these algorithms to the optimal 
solution is established, but the rate of convergence is explored only numerically. \cite{HLRO11} proves the asymptotic optimality of a myopic 
inventory control policy based on the Kaplan-Meier estimator, and compares the performance that it achieves to previously proposed policies through extensive computational experiments and for various classes of distributions. Finally, in the related setting of a repeated 
stochastic capacity control problem, \cite{vRM00} proposes a stochastic approximation algorithm, whereas \cite{ME15} introduces a methodology 
based on maximum entropy distributions. Similarly to all the works in this thread, the asymptotic convergence of these algorithms is established 
but the rates of convergence are not addressed.

The papers that come closest to our work are \cite{HR09} and \linebreak \cite{BM13}. Both consider the inventory management of a perishable 
product, in other words, 
a repeated newsvendor problem over $T$ time periods, and follow the nonparametric approach: the manager has no prior information on the 
demand - assumed to be independent and identically distributed over different time periods, drawn from an unknown distribution - and has 
access only to censored demand data. The objective is to minimize the expected regret, that is, the difference between the expected incurred 
cost and the optimal expected cost, had the demand distribution been known a priori. \cite{HR09} proposes an adaptive inventory management 
algorithm based on the methodology of online convex optimization (OCO); see \cite{Z03}. This algorithm has expected regret 
that scales as $O\left(\sqrt{T}\right)$, which is the minimax optimal scaling. This can be improved to $O\left(\log T \right)$ if the demand has 
a continuous density, uniformly bounded away from zero. We note that the stationarity of the underlying demand is not crucial for OCO,\footnote{In fact, 
\cite{Z03} does not make this assumption.} but the fact that demand and inventory are continuous quantities is important: 
this is a gradient descent-type algorithm, and the continuity of state and action spaces implies that a direction of cost improvement is available almost surely, 
irrespective of censoring. Consequently, in the case of discrete demand, their methodology requires the existence of a lost-sales indicator 
to recover the $O\left(\sqrt{T}\right)$ scaling of the expected regret. As we illustrate in our numerical experiments, in the absence of such an indicator,
this approach cannot guarantee sublinear regret.

On the other hand, the objective in \cite{BM13} is to understand the impact of the available information/feedback structure (fully observable/censored/partially 
censored demand) on the optimal scaling of the expected regret.\footnote{Relatedly, following the parametric approach, \cite{BG15} and \cite{JRW15} compare the 
optimal inventory management policies under different information levels (observable lost sales/unobservable lost sales but observable stock-out times/unobservable 
lost sales and stock-out times), and conclude that through additional information, improved performance can be achieved. In a different direction 
but still in a parametric setting, \cite{M15} studies the impact of inventory record inaccuracy when combined with censored demand samples.} 
In the case of discrete demand and censored observations, which is most relevant to our work and often the case in practice, the authors 
develop an algorithm based on alternating exploration and exploitation intervals whose expected regret scales as $O\left(\log T \right)$, 
which is the minimax optimal scaling.\footnote{The discrepancy between the results in \cite{HR09} and \cite{BM13} in the case of discrete demand 
stems from the fact that the latter study makes somewhat stronger assumptions regarding the families of distributions allowed. In 
particular, it is assumed that the expected cost function is not ``too flat'' around the optimal ordering quantity, that is, the separation between the 
optimal and the best suboptimal ordering quantity is bounded from below by some $\epsilon$, and the upper bound on the 
expected regret of their algorithm also scales as $O\left(1/\epsilon^2\right)$.} Overall, the proposed algorithm is well equipped to deal with the 
discreteness of the demand, and the issues it may create when combined with censoring, but the stationarity of the underlying demand 
process seems to be crucial here: the exploration intervals could lead to very poor inferences, and hence performance during the subsequent 
exploitation intervals, if the demand is nonstationary.

Finally, in a related setting and with the same motivation as our work, is the recent contribution \cite{BGZ15}. The authors develop a modeling 
and algorithmic framework for non-stationary stochastic optimization, an adaptation of the OCO setting where: (i) the actions 
of the nature/adversary are constrained by a variational budget; (ii) the comparator sequence is not static (i.e., a fixed action) but rather a dynamic 
oracle, a feature that bodes well with their focus on non-stationary environments. One of the main insights of the paper is to show that algorithms 
that perform well for OCO, can be successfully adapted for non-stationary stochastic optimization.
On the other hand, the theory is developed in a quite abstract setting, so it does not take advantage of the special structure 
of the inventory management problem at hand. Moreover, and quite importantly, extending their approach to settings with limited feedback, 
such as censored data, is a challenging task, e.g., \cite{BGZ14} explores the case of bandit feedback. As discussed earlier in the approach 
of \cite{HR09}, OCO-based algorithms may not have guaranteed performance if a direction of cost improvement is not always available, e.g., 
through a lost-sales indicator.

Our goal is to fill the conceptual gap that exists between these influential works, that is, to devise an inventory management policy that 
learns from censored data without making any parametric assumptions, and which has guaranteed performance under discrete and 
nonstationary\footnote{It may be worthwhile to mention a thread in the literature that develops 
online approximation algorithms for inventory management problems, which are based on marginal cost accounting and cost balancing, and have 
guaranteed worst-case performance; see \cite{LPRS07}, \cite{LJN08}, and several follow-up works. While their framework allows for correlated 
and nonstationary demand processes, which is a big part of our motivation, there is no learning, and hence no censoring problem in 
their case: given a realization of past demand and possibly additional information, a conditional joint distribution of future 
demand is assumed to be available.} demand; in fact, near-optimal performance in terms of the regret criterion. Both features are quite 
important in practice, and as we elaborate below, jointly, they require a different methodological approach than the ones existing in the 
literature on stochastic inventory theory. Moreover, in existing works, action space (orders) and outcome space (demand) coincide. In practice, 
however, there may be only few, predetermined ordering levels, for example, due to fixed ordering cost. Accordingly, we disentangle the two, 
and provide more refined results that highlight the scaling of the expected regret not only with respect to the number of time periods, but also 
with respect to the number of ordering decisions available and the size of the demand support. We note that our approach has guaranteed 
performance not only with respect to the standard notion of regret, which is based on the best fixed action, in hindsight, but also with respect 
to the tracking regret, a much stronger benchmark. Furthermore, we discuss the possibility of extending our approach to even stronger benchmarks, 
such as the dynamic regret of \cite{BGZ15}.


\bigskip\subsection{A Nonstochastic Approach to Dynamic Learning Problems} \label{OnLearn}

 To understand why, and how, our methodological approach deviates from previous literature on stochastic inventory theory, let us put the 
repeated newsvendor problem with demand learning in a broader context, that of \textit{sequential prediction}: a forecaster observes one 
after the other the elements of a sequence of outcomes $y_1, y_2, \ldots$ Before the $t^{th}$ element of the sequence is revealed, 
the forecaster predicts its value $y_t$ on the basis of the previous $t-1$ observations. In the ``classical'' statistical theory of sequential 
prediction, the sequence of outcomes is assumed to be a realization of a stationary stochastic process. Under 
this hypothesis, statistical properties of that process may be estimated on the basis of the sequence of past observations, and effective 
prediction rules/policies can be derived from these estimates.\footnote{For instance, in the context of the repeated newsvendor model with 
time-stationary demand whose underlying distribution is unknown, one would use past observations to estimate the critical quantile of the 
demand distribution, which is the optimal solution when that distribution is known.} In such a setting, the performance of a policy is 
captured through the expected value (with respect to the probability distribution governing the process) of some loss function measuring the 
discrepancy between predicted values and true outcomes.

On certain occasions though, an underlying probabilistic structure (in particular, time stationarity) may be hard to justify or estimate. Then, an 
alternative approach is to view the sequence $y_1, y_2, \ldots$ as the product of some unknown and unspecified mechanism. In lack of any 
probabilistic assumptions on the sequence, the goal is to come up with prediction rules/policies that perform reasonably well for every possible 
outcome sequence, that is, to predict \textit{individual sequences} uniformly well. The lack of probabilistic assumptions also raises the 
question of how to quantify the forecaster's performance. To provide a baseline for measuring performance in this setting, one may 
introduce a class of reference forecasters or ``experts.'' These experts make their predictions available to the forecaster before each 
outcome is revealed, and the forecaster makes her own prediction based on the advice of the different experts. The goal of the forecaster is 
to keep her cumulative loss close to that of the best expert in hindsight (i.e., with full knowledge of the entire sequence of outcomes), or in 
other words to minimize her \textit{regret}, uniformly over all outcome sequences.

The literature on prediction of individual sequences originates from repeated games: several of the basic ideas are 
introduced in the early influential works \cite{B56} and \cite{H57}, such as the use of randomization as a powerful tool to achieve low regret 
when it would be impossible otherwise. In a different strand of literature, \cite{C65}, \cite{LZ76}, and \cite{ZL77} give the 
information-theoretic foundations of sequential prediction of individual sequences, motivated by the problems of data compression and ``universal'' coding. More 
recently, the prediction of individual sequences has become a topic of intense research activity in the subfield of machine learning termed 
\textit{online learning}. The pioneering papers \cite{dSMW88}, \cite{LW89}, and \cite{V90} illustrate how the framework of prediction with 
expert advice can be transformed into a model of online learning, and a plethora of subsequent works builds on and extends their results; for 
a comprehensive treatment of the topic the reader is referred to \cite{CBL06}. Let us note that in this setting, typically, each expert is 
associated with a different action (e.g., predicted outcome), so the regret of the learner (e.g., the forecaster) is measured against the 
best fixed action in hindsight.

Of particular interest to us is the literature on \textit{partial monitoring}, a special case of the online learning paradigm where the 
information available to the learner is limited in some way. A notable member of this class is the (nonstochastic) multi-armed bandit problem 
studied in \cite{ACFS02}, where the actual outcome is not revealed to the learner after each round. What the learner knows with 
certainty is the loss of the actions that she takes, but she has no information on the losses of other actions she could have chosen instead. Hence, she 
has to ``explore'' in order to learn the losses associated with different actions, and to ``exploit'' by converging sooner rather than later 
to the ones she believes have the smallest loss. The main result is that a simple randomized policy, termed the \textit{Exponentially 
Weighted Forecaster} (EWF), achieves expected regret\footnote{We note that when one refers to expected regret in the context of prediction 
of individual sequences, the expectation is taken with respect to the randomization induced by the forecasting policy; in essence, with 
respect to the beliefs of the forecaster. This is because in a nonstochastic setting there is no ground-truth or benchmark distribution 
to compare against.} that scales as $O\left(\sqrt{T}\right)$, which is also the optimal scaling for the particular problem. 
\citet{PS01} extends this setting significantly, to a repeated game between a learner and an opponent, making sequential decisions in a 
finite action and outcome space, respectively, with the loss and the feedback that the learner receives at the end of each round being 
arbitrary (but time invariant) functions of the action and outcome chosen during that round. The authors prove that under a relatively mild 
technical assumption, an adaptation of the EWF achieves expected regret that scales as $O\left(T^{3/4}\right)$. \citet{CLS06}
improves this bound to $O\left(T^{2/3}\right)$, and also provides a lower bound of the same order for a specific feedback/loss structure. The 
question of the exact dependence of the minimax regret on the problem structure was first tackled in \citet{BPS11}; see also \citet{ABPS13} for the case of 
stochastic outcomes. This series of research efforts culminated in \cite{BFPRC14}, where a complete classification of finite 
partial monitoring games is provided: ``trivial'' games where the expected regret does not scale with $T$; ``easy'' 
games where the scaling is $\Theta\left(\sqrt{T}\right)$; ``hard'' games where the scaling is $\Theta\left(T^{2/3}\right)$; and ``hopeless'' 
games with scaling $\Theta\left(T\right)$. Importantly, there can be no other scaling apart from the aforementioned four, despite the fact 
that the structure of the game, in terms of both the loss and the feedback functions, can be chosen arbitrarily.
A geometric condition, termed local observability, is shown to distinguish ``easy'' from ``hard'' games, and generic algorithms 
are developed for each case, albeit significantly more complicated than the EWF.
 
It is worthwhile to mention the sole paper on the repeated newsvendor problem with demand learning that adopts the nonstochastic 
framework, \cite{LLMNV10}. In their setting, however, there is no censoring, while demand and orders are continuous quantities; both crucial 
features as explained earlier.


\bigskip\subsection{Main Contributions}  \label{Contrib}

As already mentioned above, we focus on the multi-period inventory management problem of a perishable product, that is, a repeated newsvendor problem, where no information on the demand is available a priori, and learning occurs via sales/censored data. Importantly, demand and orders are discrete quantities, and no probabilistic assumptions (in particular, independence or time stationarity assumptions) are made regarding the demand. Our main contributions can be summarized as follows.
\begin{itemize}
\item[(1)] We show that the simple EWF policy has expected regret that scales as $O \left(\sqrt{T \log T}\right)$, which is optimal 
up to the logarithmic term. Note that even in the case where the demand over 
different time periods is i.i.d., no better scaling than $\Omega \left(\sqrt{T}\right)$ can be achieved, unless further assumptions are 
made on the demand; see Section 2.5 in \cite{HR09} and Section 2.3 in \cite{BM13}. 
\item[(2)] We disentangle the impact of the cardinality of the 
action space (orders) from that of the outcome space (demand). In particular, we assume that there are $N$ ordering points that the inventory 
manager can choose from at every period, arbitrarily placed within the demand support. We show that the expected regret of the EWF policy 
scales as $O\left(\log N \right)$, that is, near-optimally. Notably, the general theory developed in \citet{B13} and \citet{BFPRC14} guarantees only polynomial 
scaling with respect to $N$. Crucial to the performance of the EWF policy is a carefully designed cost estimator, which leverages 
the special structure of the problem.
\item[(3)] The previous results allow us to reach an important conclusion about the dynamic learning problem at hand: the benefit from 
``information stalking'' as well as the cost of censoring are both negligible, at least with respect to the regret criterion. Intuitively 
speaking, this seems to extend and solidify the insights in \cite{BCMpre}, where in a time-stationary setting and for the class of 
``newsvendor'' distributions, the benefit from ``information stalking'' has been shown to be very small, and the cost of censoring not 
significant either.
\item[(4)] We modify the proposed policy so that it has guaranteed performance with respect to the tracking regret, that is, using as 
benchmark the best sequence of inventory decisions that switches a limited number of times. The tracking regret is a much stronger 
benchmark, particularly suitable for nonstationary demand models as our numerical experiments suggest. The tradeoff is a somewhat looser 
bound: while the scaling of the expected tracking regret with respect to $T$ and $N$ remains optimal up to logarithmic terms, the upper 
bound now includes a multiplicative term that relates to the number of times that the reference sequence is allowed to switch actions.
\item[(5)] We consider the ``combinatorial'' version of the repeated newsvendor problem, that is, single-warehouse multi-retailer inventory management of a perishable product, with facility-dependent fixed ordering costs and overage/underage cost rates. We extend our approach so that, again, it achieves expected regret that scales as $O \left(\sqrt{T \log T}\right)$ and $O\left(\log N \right)$, respectively, that is, near-optimally in both cases.
\end{itemize}


\bigskip\subsection{Outline of the Paper}

The remainder of the paper is organized as follows. Section 2 provides a detailed description of our benchmark model, a repeated newsvendor 
model with demand learning via censored data, gives some necessary background on online learning, and presents the proposed inventory 
management policy accompanied by a regret analysis. Section 3 introduces the notion of tracking regret, and shows that a modification of the 
proposed policy has guaranteed performance with respect to the latter criterion. This is followed by extensive numerical experiments in 
Section 4. Section 5 presents the ``combinatorial'' version of the repeated newsvendor problem, that is, single-warehouse multi-retailer 
inventory management of a perishable product, again incorporating demand learning through censored observations from a nonparametric 
viewpoint. We conclude the paper with a brief commentary in Section 6. All major proofs are relegated to an appendix, at the end of the 
paper.


\bigskip\section{Inventory Management with Censored Demand Data}

\subsection{Problem Formulation}

Fix $T, D \in \N$, and define the sets $\T = \{1, 2,\ldots, T\}$ and $\D = \{0, 1,\ldots, D\}$. Also, fix $N \in \N$, with $N \leq |\D|=D+1$, and let $\I$ be an arbitrary subset of $\D$ with cardinality $N$. We denote by $\mathbb{I}_E$ the indicator variable of event $E$, and by $(x)^+$ the maximum of scalar $x$ and 0.

Consider a firm that sells a single perishable product to the market. The following ``game'' between the firm's inventory manager and the 
market is repeated over $T$ time periods: at the beginning of period $t \in \T$, the inventory manager chooses an inventory level $I_t \in 
\I$ to have in stock. (So, implicitly, we assume that there is zero lead time between placing and receiving an order.) Simultaneously, the 
market chooses the demand that the firm experiences during that period, $d_t \in \D$, which is covered up to the extent that the available 
inventory allows. At the end of the period any remaining inventory perishes, and the firm incurs a cost
\begin{equation}
c(I_t,d_t) = h (I_t - d_t)^+ + b (d_t - I_t)^+, \label{newscost}
\end{equation}
where $h, b>0$ represent the overage and underage cost rates, respectively, which are known to the manager and fixed for all time periods.

An important characteristic of our model is that the inventory manager has no information about the demand prior to the beginning of the game. Moreover, she has to ``learn'' the demand via \textit{censored data}: at the end of period $t$, the inventory manager can only observe the sales during that period, $\min\left\{I_t,d_t\right\}$. In particular, if the inventory $I_t$ turns out to be less than or equal to the demand, then the inventory manager does not know with certainty the exact demand that the firm experienced, $d_t$, nor the exact cost that it incurred, $c(I_t,d_t)$.

Fix a sequence of demand realizations $\{d_t\}$. We define the \textit{regret} of the inventory manager, for any sequence of inventory 
decisions $\{I_t\}$, to be the difference between the cumulative cost that is actually incurred and the cost that would have been incurred under 
the best fixed inventory decision, in hindsight:
\[
\R(T) = \sum_{t \in \T} c(I_t,d_t) - \min_{i \in \I} \sum_{t \in \T} c(i,d_t),
\]
where $c(i,d_t)$ is defined similarly to Eq.\ \eqref{newscost}. We denote by $i^*$ the minimizer in the equation above, omitting its dependence on the sequence $\{d_t\}$ for convenience.

An important point of our work is that when no probabilistic assumptions are made about the demand, that is, if one adopts the so-called 
\textit{nonstochastic} viewpoint, then in many cases randomization is the only way to achieve low regret. Of course, under a randomized 
inventory management policy the regret is a random variable, so our goal is to design policies that have low expected regret:
\begin{equation}
\E[\R(T)] = \EE{\sum_{t \in \T} \sum_{i \in \I} p_i(t) \big(c(i,d_t) - c(i^*,d_t)\big)}, \label{expreg}
\end{equation}
where $p_i(t)$ denotes the probability of selecting inventory level $i \in \I$ at the beginning of time period $t \in \T$, conditional on 
all previous decisions made, that is, $p_i(t) = \P(I_t = i \mid I_1, I_2,\ldots,I_{t-1})$.


\bigskip\subsection{Feedback Structure and Local Observability}

Before we proceed to the presentation and performance analysis the proposed inventory management policy, we delve into the 
feedback structure of the problem and make a connection to the notion of \textit{local observability}, as first defined by \citet{BPS11}. 
Understanding and exploiting this structure is precisely what allows for good (near-optimal) scaling of the expected regret with respect to 
$T$.

Consider an arbitrary pair of inventory decisions $i, j \in \I$, and a demand realization $d \in \D$. We wish to compute the difference between the cost of the two actions. Without loss of generality, assume that $i>j$. (If $i=j$ then, obviously, the difference in cost is zero.) We have that:

\begin{itemize}
\smallskip\item[(i)] $c(i,d) - c(j,d) = h (i-j)$, if $d \leq j$;

\smallskip\item[(ii)] $c(i,d) - c(j,d) = h (i-d) - b (d-j) = h i + b j -(h+b) d$, if $j < d < i$;

\smallskip\item[(iii)] $c(i,d) - c(j,d) = b (j-i)$, if $d \geq i$.
\end{itemize}

On the other hand, let $k \in \N$, with $k \leq D$, and denote by $\mathbb{L}_k$ the $k \times k$ identity matrix, and by $\mathbb{M}_k$ the $k \times (D-k)$ matrix, where $\mathbb{M}(i,j)=1$ if $i=k$, and 0 otherwise. Finally, let $e_d$ be the $D+1$-dimensional column vector, with $e_d(j)=1$ if $j=d$, and 0 otherwise.

The \textit{signal matrix} of inventory decision $i \in \I$, denoted by $\S_i$, is a $(i+1) \times (D+1)$ matrix, where element $\S_i(k,j)=1$ if the sales of the firm are equal to $k \in \{0, 1,\ldots, i\}$, assuming that the demand is equal to $j \in \D$ and the inventory is equal to $i$, and 0 otherwise. It can be verified that $\S_i$ is equal to the concatenated matrix:
\[
\S_i = \big[\mathbb{L}_{i+1} \mid \mathbb{M}_{i+1}\big].
\]

The essence of local observability is that the difference in cost between any two inventory decisions, for any demand realization, can be expressed in terms of their signal matrices, that is, the information that the inventory manager receives in the respective cases via the firm's sales.

\medskip\begin{lemma} \label{locobs}
(Local Observability) Let $i,j \in \I$ be arbitrary inventory decisions. There exist vectors $v_i\in \real^{i+1}$ and $v_j\in \real
^{j+1}$ such that
\[
\left(v_i^{T} \S_i - v_j^{T} \S_j\right) e_d = c(i,d) - c(j,d), \qquad d \in \D.
\]
\end{lemma}

\begin{proof}
Consider the $(i+1)$-dimensional column vector $v_i$, where $$v_i(k)=h i - (h+b)(k-1),\qquad k \in \{1, 2,\ldots, i+1\}.$$ Define similarly the vector $v_j$. The result follows through straightforward calculations.
\end{proof}

Lemma \ref{locobs} implies that the game between the inventory manager and the market is locally observable, in the sense of Definition 6 in \cite{BFPRC14}. This classifies the repeated newsvendor problem with demand learning via censored data as an ``easy'' partial monitoring problem - see Section \ref{OnLearn} - which implies that the correct scaling of the expected regret is $\Theta \left(\sqrt{T}\right)$. Below, we introduce a simple policy that is near-optimal with respect to this criterion.


\bigskip\subsection{The Exponentially Weighted Forecaster}

The Exponentially Weighted Forecasting (EWF) is a well-studied online learning methodology that simultaneously ``explores'' and 
``exploits,'' in a randomized way. The main idea behind it is to keep track of not only the cost of actions that are actually taken, but 
also of the estimated cost of all other actions that could have been taken instead. Of course, the specifics of cost estimation are 
context-specific, as they are closely tied to the type of feedback that the learner receives. Based on the cumulative estimated cost of the 
different actions, the learner forms beliefs about the chances each of them has being the best one, in hindsight, and prioritizes future 
actions accordingly.

More concretely, let $\wt c(i,d_t)$ be the estimated cost that inventory decision $i \in \I$ would have incurred at period $t \in \T$ 
under demand $d_t$. Note that, implicitly, $\wt c(i,d_t)$ may also be a function of the actual inventory decision $I_t$ that was made at period
$t$. In fact, that is the case in the cost estimator that we propose below.
Similarly, we define $\wt C_i(t)$ as the cumulative estimated cost of (fixed) inventory decision $i \in \I$ at period 
$t \in \T$, with $\wt C_i(0)=0$. The cumulative estimated cost can be computed through the recursion:
\[
\wt C_i(t) = \wt C_i(t-1) + \wt c(i,d_t), \qquad i \in \I.
\]

For convenience, let us also define $W_i(t) = e^{-\eta \wt C_i(t)}$ and $W(t) = \sum_{i \in \I} W_i(t)$, where $\eta$ is a positive 
constant whose exact value depends on the primitives of the problem in a way that is specified later on. Using this notation, we have that
\begin{equation}
W_i(t) = W_i(t-1) e^{-\eta \wt c(i,d_t)}, \qquad i \in \I. \label{dyn}
\end{equation}

The EWF policy chooses inventory $I_t = i$ with probability
\begin{equation}
p_i(t) = (1-\gamma) \frac{W_i(t-1)}{W(t-1)} +\frac{\gamma}{N}, \qquad i \in \I, \label{ewf}
\end{equation}
where $\gamma$ is another parameter, in the $(0,1)$ interval, whose precise value will be determined later.

Note that the EWF policy simultaneously ``explores'' the available action space by making every inventory decision with probability at 
least $\gamma / N$, and ``exploits'' by assigning higher probability to decisions that have low cumulative estimated cost. The precise way 
that the inventory manager prioritizes between exploration and exploitation depends on the exact values of the $\eta$ and $\gamma$ parameters.

While the EWF is a generic and well-studied policy, what takes advantage of the special structure of the problem at hand is the design of 
the proper cost estimator $\wt c(i,d_t)$. To get some insight into what type of estimator may be suitable, let us assume that at period 
$t \in \T$ the inventory manager decides to hold inventory $I_t$. At the end of period $t$, the firm gets (potentially censored) feedback 
about the demand, that is, the sales $\min\{I_t,d_t\}$. Importantly, this feedback also gives information about the sales that the firm would have 
had during the particular period, had the inventory manager chosen any $i \leq I_t$:
\begin{itemize}
\smallskip\item[(i)] if the feedback was censored, that is, $d_t \geq I_t$, then $\min\{i,d_t\}=i$, for all $i \leq I_t$;

\smallskip\item[(ii)] if the feedback was not censored, that is, $d_t<I_t$, then the demand $d_t$ is known with certainty and the sales 
$\min\{i,d_t\}$ can be computed, for all $i \in \I$.
\end{itemize}

\smallskip We use this insight to define the estimated cost of action $i$ under demand $d_t$ as follows:
\begin{equation}
\wt c(i,d_t) = \frac{\II{I_t \geq i}}{\P_t(I_t \geq i)} \left(v_i^{T} \S_i e_{d_t} + \beta \right),  \qquad i \in \I, \label{est}
\end{equation}
where $v_i$ and $\S_i$ are taken from Lemma \ref{locobs}, $\beta = D \cdot \max\{h,b\}$, and $\P_t(I_t \geq i) = \sum_{j \in \I: j \geq i} 
p_j(t)$ according to Eq.\ \eqref{ewf}.

Next, we provide some properties regarding the bias and variance of the above estimator, which are critical in the performance analysis that follows.

\medskip\begin{lemma} \label{estprop}
(Bias and Variance of the Estimator) The cost estimator in Eq.\ \eqref{est} satisfies:
\[
\E_t \left[\wt c(i,d_t)\right] = v_i^{T} \S_i e_{d_t} + \beta,
\]
so that $\E_t \left[\wt c(i,d_t)\right] \in (0,2 \beta)$, and
\[
\E_t \left[\wt c(i,d_t)^2\right] \leq \frac{4 \beta^2}{\P_t(I_t \geq i)},
\]
where $\P_t(\cdot)$ and $\E_t[\cdot]$ respectively denote the probability and the expectation conditioned on the history of interaction up 
until the beginning of round $t$.
\end{lemma}

\begin{proof}
The first part of the lemma follows directly by noting that
\[
\E_t \left[\wt c(i,d_t)\right] = \frac{\E_t \left[\II{I_t \geq i}\right]}{\P_t(I_t \geq i)} \left(v_i^{T} \S_i e_{d_t} + \beta 
\right) = v_i^{T} \S_i e_{d_t} + \beta.
\]
The fact that $\E_t \left[\wt c(i,d_t)\right] \in (0,2 \beta)$ is a direct consequence of $\beta$ being an upper bound on the absolute value of $v_i^{T} S_i e_{d_t}$, for every $i \in \I$ and $d_t$. Hence, regarding the second part of the lemma, we have that
\[
\E_t \left[\wt c(i,d_t)^2\right] = \frac{\E_t \left[\II{I_t \geq i}\right]}{\P_t(I_t \geq i)^2} \big(v_i^{T} \S_i e_{d_t} + \beta 
\big)^2 \leq \frac{4 \beta^2}{\P_t(I_t \geq i)}.
\]
\end{proof}

\medskip Note that the proposed estimator is biased, as $\E_t \left[\wt c(i,d_t)\right] \neq c(i,d_t)$. In particular, the estimator is 
pessimistic in the sense that it always overestimates the actual cost incurred. A direct corollary of Lemmas \ref{locobs} and 
\ref{estprop} is the following result.

\medskip\begin{lemma} \label{unbias}
The cost estimator in Eq.\ \eqref{est} is unbiased when inferring the difference in cost between two actions:
\[
\E_t \left[\wt c(i,d_t) - \wt c(j,d_t)\right] = c(i,d_t) - c(j,d_t), \qquad i, j \in \I. 
\]
\end{lemma}

\medskip The significance of Lemma \ref{unbias} lies in the fact that the regret, by definition, is a metric that is based on cost differences. This facilitates the analysis in our main result, which characterizes the performance of the proposed inventory management policy with respect to the regret criterion.

\medskip\begin{theorem} \label{ewfthm}
Consider the repeated newsvendor problem described above. The expected regret of the EWF policy with the cost estimator in Eq.\ \eqref{est} and parameters
\[
\gamma=\frac{1}{2 \beta T}, \qquad \eta = \sqrt{\frac{\log N}{4 \beta^2 T \log \left(2 \beta T N^3 + N + 2 \right)}},
\]
is bounded from above as follows:
\[
\E[\R(T)] \leq 4 \beta \sqrt{T \log N \log \left(2 \beta T N^3 + N + 2 \right)} + 2 \beta \sqrt{T \log N} + 1.
\]
\end{theorem}

\begin{proof}
See Appendix 1.
\end{proof}

\medskip Theorem \ref{ewfthm} implies that the expected regret of the EWF policy scales as $O \left(\sqrt{T \log T}\right)$. On the other hand, the expected regret of 
any policy in the particular setting scales as $\Omega \left(\sqrt{T}\right)$, even if the demand over different time periods is i.i.d.; see 
Section 2.5 in \cite{HR09} and Section 2.3 in \cite{BM13}.

Moreover, the expected regret of the EWF policy scales as $O (\log N)$. While the correct scaling with respect to $N$ is not known, the EWF 
policy cannot be further than a logarithmic factor off the optimal. We note that the best known scaling of the expected regret for locally observable,
partial monitoring problems is $O\left(\sqrt{N}\right)$, achieved by the algorithm introduced in \citet{B13}, a scaling that we
improve considerably upon by taking advantage of the special structure of the problem when designing the cost estimator.

Note that since $\beta = D \cdot \max\{h,b\}$, the expected regret of the EWF policy scales as $O \left(D\sqrt{\log D}\right)$. It can be 
easily verified that the expected regret of any policy scales as $\Omega(D)$ so, again, the performance achieved by the EWF policy is 
near-optimal.

Finally, observe that the suggested choice of the parameters $\eta$ and $\gamma$ involves the total number of
time periods $T$. Thus, in order to implement the suggested policy, the inventory manager needs to know $T$--or at least
an estimate of it--in advance. On the other hand, there are standard ways of getting rid of this assumption in the literature of online learning.
In order to avoid tedious but straightforward technicalities, we assume that the inventory manager knows $T$
in advance and refer the reader to Section~2.3 of \cite{CBL06}.


\bigskip\subsection{Benefit from ``Information Stalking'' and Cost of Censoring}

``Information stalking,'' that is, the additional exploration that an optimal policy performs in a dynamic learning setting compared to 
reasonable myopic policies, can be measured in our scenario by the right-most term in Eq.\ \eqref{ewf}, which captures the frequency of 
purely exploratory decisions made by the EWF policy.\footnote{The other term on the right-hand side of Eq.\ \eqref{ewf} captures the beliefs of the 
forecaster about each fixed action being the best one, in hindsight. So, the randomization that it induces is not equivalent to exploration. 
It is, rather, due to the absence of an underlying probabilistic structure.} Note that with the optimal selection of the $\gamma$ parameter, 
this term scales like $1/(TN)$. So, roughly speaking, across all periods and inventory decisions, the proposed policy is expected to explore only a constant number on times, differing little in that sense from a myopic policy. Moreover, the amount of exploration 
can be reduced arbitrarily: by choosing $\gamma = 1/T^k$, for any $k>1$, the number of expected exploratory decisions is decreasing in $k$, 
at the cost of a constant term (so, not affecting the scaling) in the upper bound on the expected regret.

Let us also remark on the cost of censoring, that is, the additional cost incurred by having censored observations instead of pure demand 
samples, as captured by the regret criterion. If there is no censoring in the demand, then one can still use the EWF policy, simply 
replacing the estimator in Eq.\ \ref{est} with $c(i,d_t)$, the actual cost that would have been incurred if the 
manager had held inventory $i\ \in \I$ at period $t\ \in \T$. By following similar arguments to the proof of Theorem\ \ref{ewfthm}, it can be verified that the 
expected regret of the EWF policy in that case scales as $O \left(\sqrt{T}\right)$, $O \left(\sqrt{\log N}\right)$, and $O \left(D\right)$, 
respectively, in terms of the three key primitives of our problem. We note that this setting falls into the class of online learning 
problems with full information, and matching lower bounds are known: for a cost function that is defined as the absolute value of the 
difference between action and outcome, so essentially the newsvendor cost function in Eq.\ \eqref{newscost}, no policy can achieve scaling 
of the expected regret that is better than the scaling achieved by the EWF policy, for all three key primitives of the problem; see Theorem 
3.7 in \cite{CBL06}. Comparing these to Theorem\ \ref{ewfthm}, we have that censoring does not cost more than a root-log factor off the 
optimal scaling, for all three primitives. In that sense, the cost of censoring is negligible with respect to the regret criterion.


\bigskip\section{Tracking Regret and the Fixed-Share Forecaster} \label{TrackReg}

In this section, we consider a stronger notion of regret that compares the total cost incurred by a given policy to that of an arbitrary
sequence of inventory decisions. Of course, establishing non-trivial performance guarantees is only possible under some restrictions on either 
the sequence of demands, or the reference sequence of decisions. Following the work of \citet{HW98}, we consider reference sequences that 
switch between decisions at most $S$ times. Clearly, we cannot expect strong guarantees for values of $S$ comparable to $T$, 
so we focus on the case where $S\ll T$, which is also the more relevant in practice. Achieving low regret against such comparators intuitively 
translates to good performance in nonstationary environments, where the demand distribution may change abruptly several times, but otherwise 
remains roughly stationary for long periods of time. 

More formally, let $i_{[T]}=\pa{i_1,i_2,\dots,i_T}\in \I^T$ be a sequence of inventory decisions, and let
\[
C\pa{i_{[T]}} = \sum_{t \in \T} \II{i_{t}\neq i_{t+1}}
\]
be the complexity of that sequence, that is, the number of times that $i_{[T]}$ switches 
between two actions. We denote the class of sequences of complexity at most $S$ by
\[
 \I^T_S = \pa{i_{[T]}=\pa{i_1,i_2,\dots,i_T}: C(i_{[T]})\le S}.
\]
Then, we can define the \textit{tracking regret} against class $\I^T_S$ as
\[
\R_{S}(T) = \sum_{t \in \T} c(I_t,d_t) - \min_{i_{[T]} \in \I^T_S} \sum_{t \in \T} c(i_t,d_t).
\]

We propose a simple variant of the EWF policy that aims to minimize the tracking regret. This variant is based on the 
\textit{Fixed-Share Forecaster} (FSF) introduced in \citet{HW98}; see also \cite{ACFS02}, \citet{BW02}, and \citet{CBGLS12}. The key 
difference compared to the standard EWF is that instead of using Equation~\eqref{dyn} for updating the weights, the 
FSF policy uses the update rule:
\begin{equation}\label{eq:fsfupdate}
 W_i(t) = W_i(t-1) e^{-\eta \wt c(i,d_t)} + \frac{\alpha}{N} \sum_{j\in\I} W_j(t-1),
\end{equation}
for all $i$ in $\I$, where $\alpha>0$ is an suitably chosen constant. Otherwise, the probabilities $p_i(t)$ and the cost estimates 
$\tc(i,d_t)$ are computed as described in the previous section. The following result summarizes the performance guarantee that
we can prove for the FSF policy.

\medskip\begin{theorem} \label{fsfthm}
Consider the repeated newsvendor problem. The expected tracking regret of the FSF policy with the cost estimator in Eq.~\eqref{est} 
and parameters
\[
\alpha = \frac{1}{T}, \qquad \gamma=\frac{1}{2 \beta T}, \qquad \eta = \sqrt{\frac{\log \pa{NT}}{4 \beta^2 T \log \left(2 \beta T N^3 + N + 
2 \right)}},
\]
is bounded from above, for any $S$, as follows:
\[
\E[\R_S(T)] \leq 2 (S+1) \beta \sqrt{T \log \pa{NT} \log \left(2 \beta T N^3 + N + 2 \right)} + 2 \beta \sqrt{T \log N} + 2.
\]
Furthermore, for fixed $S$, choosing parameters
\[
\alpha = \frac{1}{T}, \qquad \gamma=\frac{1}{2 \beta T}, \qquad \eta = \sqrt{\frac{S	\log \pa{NT}}{4 \beta^2 T \log \left(2 \beta T N^3 + 
N + 2 \right)}},
\]
leads to an improved bound:
\[
\E[\R_S(T)] \leq 4 \beta \sqrt{S T \log {NT} \log \left(2 \beta T N^3 + N + 2 \right)} + 2 \beta \sqrt{T \log N} + 2.
\]
\end{theorem}

\begin{proof}
See Appendix 1.
\end{proof}


\bigskip\section{Numerical Experiments} \label{Sim}

In this section, we perform a numerical study of the EWF and FSF policies, comparing them to the most relevant existing methods in 
the literature. The goal of our experiments is to illustrate the robustness of the proposed approach to nonstationary demand sequences, 
and thus to show the benefits of designing policies that do away with time-stationarity assumptions. We particularly focus on making 
comparisons to an inventory management policy based on Alternating Exploration and Exploitation phases (AEE), introduced in \citet{BM13}, 
which is designed for the same learning problem, albeit tailored to time-stationary demand settings. Our numerical investigation suggests that 
this assumption is crucial for the AEE policy to perform well: although it has excellent performance in time-stationary settings, by construction, 
it is unable to adapt to nonstationary demands, and consequently performs very poorly in the latter settings. The online gradient descent-based Adaptive
Inventory Management policy (AIM), introduced in \citet{HR09}, is also closely related to our work. This policy is shown to perform quite well
in our numerical experiments. The caveat is that, unlike other policies, it requires a lost-sales indicator to work properly in the case of discrete 
demand. As we detail in Appendix 2, without such an indicator this policy has linear regret, and may perform poorly even in settings where
the demand is constant.

In our numerical experiments, we tune the parameters as follows. Regarding the policies proposed in this paper, we set 
$\gamma = 1 / (2\beta T)$ and $\alpha = 1/T$, as suggested by Theorems~\ref{ewfthm} and~\ref{fsfthm}, and 
$\eta = \sqrt{\frac{S\log N}{4\beta^2 T}}$, with $S = 1$ for the EWF policy and some suitable $S>1$ for the FSF policy. Regarding 
the AEE policy, we use the parametrization provided in Section~5 of \citet{BM13}, without resorting to the data aggregation technique 
described in their Section~5.1. We note that, in our setting, this choice actually works in favor of the AEE policy, as aggregating 
estimates from early periods is clearly harmful when the demand sequence is nonstationary. 

In each experiment, we choose $T=100,000$ and $\I = \D = \ev{1,2,\dots,30}$. Each reported curve is an average of 100 runs, with shaded 
areas representing the standard deviations. 

Finally, to gain some insight into the cost of censoring in practice, we study both the uncensored and the censored versions of each of the 
policies described above. Note that in the full information case, the policy proposed in \citet{BM13} has no reason to explore. Rather, it
exploits in a greedy fashion by ordering the empirical critical quantile. We term the resulting policy ``greedy-full.''

We start by reproducing the first numerical experiment in Section 5 of \citet{BM13}, where the sequence of demands is i.i.d., with each 
$d_t$ generated independently from a binomial distribution representing $30$ independent trials with a success probability of $1/2$. The 
regret of each policy is plotted on Figure~\ref{fig:stat}. In this setting, the AEE policy outperforms the others by a wide margin. In particular,
in both the censored and the uncensored cases, the regret of the EWF policy is about 10 times greater than that of the AEE policy. 
This observation is not surprising: EWF is decidedly more conservative, since it aims to perform reasonably even when the demand 
is nonstationary.\footnote{The inferior performance of general no-regret algorithms, like the EWF, in time-stationary environments 
has been widely acknowledged in the online learning literature, with some remedies offered by the works of \cite{EWK14,SNL14}.}
Nevertheless, the empirical regret of the EWF policy grows in a square-root fashion, in line with our theoretical guarantees. 

\begin{figure}
\begin{center}
\includegraphics[width=.6\textwidth]{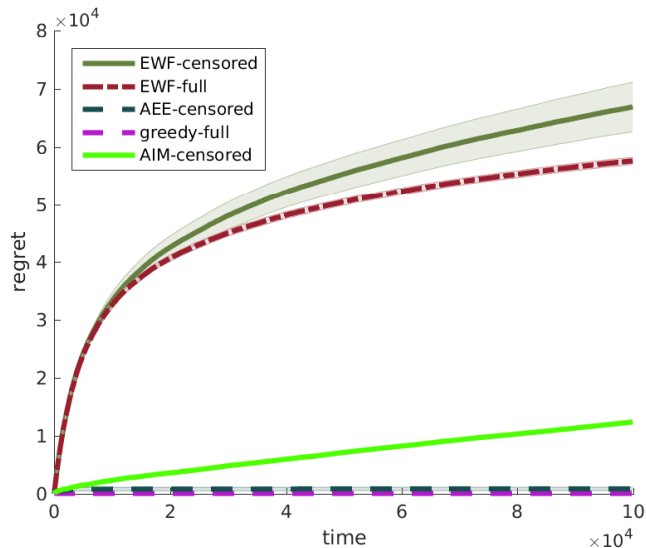}
\caption{The regret of the different policies on a stationary demand series.}\label{fig:stat}
\end{center}
\end{figure}

Our second batch of numerical experiments considers a simple nonstationary demand sequence. As before, the demands are generated 
independently from a binomial distribution with 30 trials, albeit with a time-dependent success probability $q_t$. Similarly to the previous 
experiment, we set $q_t = 1/2$ for most values of $t$, however, this probability drops to $0.1$ for $t\in[T/5,T/2]$. The performance of each 
policy is shown on Figure~\ref{fig:nonstat}; for clarity, we show the cumulative cost incurred by each policy instead of their 
respective regrets.

The first lesson from this experiment is that the AEE policy fails to cope with the shifting demand distribution, incurring a linearly increasing 
regret. The intuition behind this is simple: since the policy, effectively, collects data in deterministically (and scarcely) scheduled 
exploration periods, it is easily thrown off its tracks by a shifting demand distribution. In the particular experiment, the AEE policy bases all its decisions 
during the interval $[T/2,T]$ on data that has nothing to do with the actual demand distribution. The policy has no mechanism to 
recover from such mistakes, and introducing such a mechanism while maintaining strong performance guarantees, is far from trivial. 

In contrast, the EWF policy is robust to this nonstationary behavior. Notably, its cumulative cost in the censored case, using our carefully designed 
cost estimator, is remarkably close to that in the full-feedback case. This confirms our main insight, that censoring has minimal impact on 
the performance of well-designed learning policies in this setting. We also highlight the excellent performance of the FSF policy: ran with $S = 3$, 
that is, correctly anticipating $3$ shifts in the comparator sequence, this policy is seen to react much quicker to the distributional 
shifts compared to the standard EWF variant, thus achieving superior performance. 

\begin{figure}
\begin{center}
\includegraphics[width=.6\textwidth]{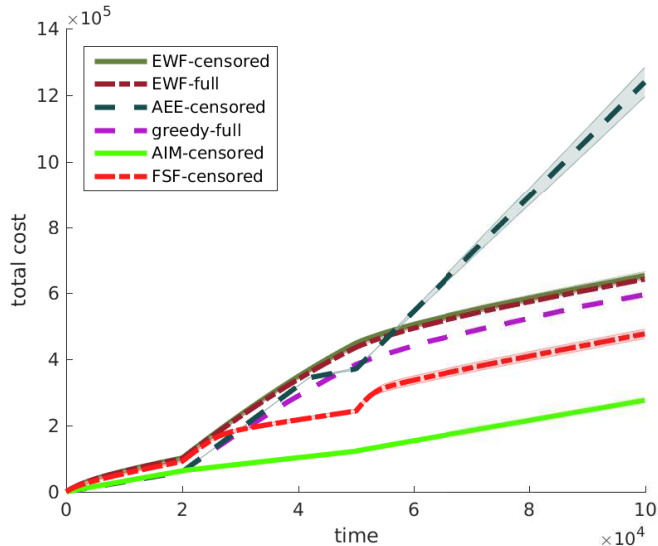}
\caption{The regret of the different policies on a nonstationary demand series.}\label{fig:nonstat}
\end{center}
\end{figure}

As a concluding remark, let us note that the performance of the AIM policy is quite good in both experiments; in fact, superior to that of the EWF/FSF 
policies. This suggests that on some occasions, it may constitute a viable option. The main disadvantage is that in the absence of performance 
guarantees (i.e., when a lost-sales indicator is not available), it is hard to know in advance how well will this policy perform. We refer the 
reader to Appendix 2 for further discussion. It is reasonable, of course, to conjecture that as the demand becomes more fine-grained, 
the value of the lost-sales indicator decreases and, thus, the AIM policy performs better; see also Section 5.2 in \cite{BM13}. However, apart
from few numerical experiments, little else is known in terms of quantifying the rate at which the discrete setting ``converges'' to a continuous one,
and the corresponding effect on the performance of online gradient descent-based policies. This simply reiterates the point about the lack of performance
guarantees.


\bigskip\section{The Single-Warehouse Multi-Retailer Problem} \label{combnews}

In this section we extend our benchmark model to include the inventory management problem of a perishable product, for a vertically 
integrated supply chain with a single warehouse and $K$ retailers. Again, the demand in each of the $K$ retailers needs to be learned from 
censored/sales data, and a nonstochastic view of the problem is adopted. As this setting has many commonalities with the model of Section 2, 
for brevity, we only present the points at which they differ.

Let us denote by $\K$ the set of retailers $\{1, 2,\ldots,K\}$. At the beginning of period $t \in \T$, the inventory manager allocates 
$I_t^{(k)} \in \I$ units of inventory from the warehouse to retailer $k \in \K$. If $I_t^{(k)}$ is not zero, then the retailer incurs a 
fixed ordering cost of $f^{(k)}$. We assume zero lead times, so the inventory is delivered to the retailer instantaneously. The retailer 
experiences demand $d_t^{(k)} \in \D$ during the particular period. At the end of the period, and depending on the initial inventory and the 
demand, the retailer incurs overage or underage cost at rates $h^{(k)}$ and $b^{(k)}$, respectively, and any remaining inventory perishes. 
Thus, at the end of time period $t \in \T$, retailer $k \in \K$ incurs a total cost of
\[
c_k\left(I_t^{(k)},d_t^{(k)}\right) = f^{(k)} \II{I_t^{(k)}>0} + h^{(k)} \left(I_t^{(k)} - d_t^{(k)}\right)^+ + b^{(k)} 
\left(d_t^{(k)} - I_t^{(k)}\right)^+.
\]

We assume that the supply chain operates in a ``push'' manner, from upstream to downstream. More specifically, the upstream supplier 
replenishes the inventory of the warehouse with $r>0$ units at the beginning of every time period. The inventory manager allocates different 
parts of this inventory to the different retailers, but may also have an incentive to keep a part of it at the warehouse, if she believes 
that the total demand at the retailer level is low. (The overage cost rate usually increases as one moves downstream.) Hence, the cost that 
the warehouse incurs over time period $t$ is equal to
\[
c_0\left(I_t^{(1)},\ldots,I_t^{(K)}, r\right) = f^{(0)} + h^{(0)} \left(r - \sum_{k \in \K} I_t^{(k)}\right)^+.
\]

The allocation that the inventory manager can make must belong to the set
\[
\A_r = \left\{ \left(i^{(1)},\ldots,i^{(K)}\right) \in \I^K: \sum_{k \in \K} i^{(k)} \leq r \right\}.
\]

The goal of the inventory manager is to minimize the total cost incurred by the warehouse and the retailers throughout the $T$ periods. 
More concretely, the demand sequences $\left\{d_t^{(k)}\right\},\ k \in \K,$ and the inventory replenishment level $r$ are exogenously 
determined by the market and the upstream supplier, respectively, and the manager wishes to minimize her expected regret, $\E[\R_c(T)]$, 
over the best fixed ($K$-dimensional) inventory decision, in hindsight.

We refer to this setting as the ``combinatorial'' version of the repeated newsvendor problem, due to the fact that the manager's inventory 
decisions have a combinatorial nature, taking values in $\A_r$. To the best of our knowledge, this setting has not been studied before from 
the angle of demand learning via censored data, and from a nonstochastic viewpoint. It may be worthwhile to mention, however, some 
high-level similarity to \cite{SCD16}, where the results in \cite{HR09} are extended to the capacitated multi-product case. The model 
analyzed in \cite{SCD16} is also a repeated newsvendor model with a combinatorial action space, but the approach is quite different: in 
order to obtain results for the (arguably harder) case of nonperishable products, strong probabilistic assumptions are made about the demand 
for the different products.

For convenience, let us define the quantities
\[
\beta = D \cdot \max\left\{ \max_{k=0,\ldots,K} \left\{h^{(k)}\right\}, \max_{k \in \K} \left\{b^{(k)}\right\} \right\},
\]
and
\[
f = \max_{k=0,\ldots,K} \left\{f^{(k)}\right\}.
\]

In what follows, we mimic the approach of Section 2, to construct an inventory management policy that performs well with respect to the 
regret criterion. At a high level, the proposed policy follows the EWF scheme, but unlike the versions discussed in previous 
sections, it draws actions in a non-uniform way during the exploration rounds.

Specifically, in each round the policy explores with probability $\gamma$. The exploration procedure 
generates a random allocation by selecting a retailer index $\kappa$ uniformly at random from $\K$, and by choosing the order level 
$I^{(\kappa)}$ uniformly at random from $\I$. The rest of the allocations can be completed arbitrarily; for simplicity, we choose an order 
level of $0$ for the remaining retailer. The probability of choosing the allocation $\left(i^{(1)},\ldots,i^{(K)}\right)$ is denoted 
by $\mu_{\left(i^{(1)},\ldots,i^{(K)}\right)}$.

Now let us describe the main component of the policy, which is computing the weights of the assignments.
Let $\wt c\left(i^{(1)},\ldots,i^{(K)}, r, d_t^{(1)},\ldots,d_t^{(K)}\right)$ be the estimated cost that inventory decision 
$\left(i^{(1)},\ldots,i^{(K)}\right)$ would have incurred at period $t \in \T$. We compute, recursively, the quantities
\[
W_{\left(i^{(1)},\ldots,i^{(K)}\right)}(t) = W_{\left(i^{(1)},\ldots,i^{(K)}\right)}(t-1) e^{-\eta \wt c\left(i^{(1)},\ldots,i^{(K)}, r, d_t^{(1)},\ldots,d_t^{(K)}\right)},
\]
where $\left(i^{(1)},\ldots,i^{(K)}\right) \in \A_r$, with $W_{\left(i^{(1)},\ldots,i^{(K)}\right)}(0)=1$. We also use the shorthand notation
\[
W(t) = \sum_{\left(i^{(1)},\ldots,i^{(K)}\right) \in \A_r} W_{\left(i^{(1)},\ldots,i^{(K)}\right)}(t).
\]

\medskip In this setting, the EWF policy chooses inventory $\left(I_t^{(1)},\ldots,I_t^{(K)}\right) = \left(i^{(1)},\ldots,i^{(K)}\right) 
\in \A_r$ with probability
\[
p_{\left(i^{(1)},\ldots,i^{(K)}\right)}(t) = (1-\gamma) \frac{W_{\left(i^{(1)},\ldots,i^{(K)}\right)}(t-1)}{W(t-1)} +\gamma \mu_{\left(i^{(1)},\ldots,i^{(K)}\right)}.
\]
Similarly to Section 2, the proper values for parameters $\eta$ and $\gamma$ are specified during the analysis stage.

The special structure of the problem is exploited by designing a suitable cost estimator. For every $\left(i^{(1)},\ldots,i^{(K)}\right) 
\in \A_r$, let
\[
\wt c_k\left(i^{(k)},d_t^{(k)}\right) = \frac{\II{I_t^{(k)} \geq i^{(k)}}}{\P_t \left(I_t^{(k)} \geq i^{(k)}\right)} 
\left(v_i^{T} \S_i e_{d_t^{(k)}} +f^{(k)} \II{i^{(k)}>0} + \beta \right),
\]
where $v_i$ and $\S_i$ are the same as in Lemma \ref{locobs}, and
\[
\P_t \left(I_t^{(k)} \geq i^{(k)}\right) = \sum_{\left(j^{(1)},\ldots,j^{(K)}\right) \in \A_r : j^{(k)} \geq i^{(k)}} 
p_{\left(j^{(1)},\ldots,j^{(K)}\right)}(t).
\]
Through this, we define the estimated cost of inventory decision $\left(i^{(1)},\ldots,i^{(K)}\right) \in \A_r$ as
\begin{equation}
\wt c\left(i^{(1)},\ldots,i^{(K)}, r, d_t^{(1)},\ldots,d_t^{(K)}\right) =  c_0\left(i^{(1)},\ldots,i^{(K)}, r\right) + \sum_{k \in \K} 
\wt c_k\left(i^{(k)},d_t^{(k)}\right). \label{estcomb}
\end{equation}

Note that the proposed estimator is designed in order to be unbiased in terms of inferring the difference in expected cost between two 
decisions:
\begin{align} \label{locobscomb}
&\E_t\left[\wt c\left(i^{(1)},\ldots,i^{(K)}, r, d_t^{(1)},\ldots,d_t^{(K)}\right) - \wt c\left(j^{(1)},\ldots,j^{(K)}, r, 
d_t^{(1)},\ldots,d_t^{(K)}\right) \right] \\ \nonumber
&=  c\left(i^{(1)},\ldots,i^{(K)}, r, d_t^{(1)},\ldots,d_t^{(K)}\right) - c\left(j^{(1)},\ldots,j^{(K)}, r, 
d_t^{(1)},\ldots,d_t^{(K)}\right), 
\end{align}
for every $\left(i^{(1)},\ldots,i^{(K)}\right), \left(j^{(1)},\ldots,j^{(K)}\right) \in \A_r$; a consequence of the local observability 
property of the feedback structure in the single-retailer setting. Furthermore, it is easy to show that the mean of the estimator satisfies
\begin{equation}
0 \leq \E_t\left[\wt c\left(i^{(1)},\ldots,i^{(K)}, r, d_t^{(1)},\ldots,d_t^{(K)}\right) \right] \leq f + K(2 \beta + f).
\label{meancomb}
\end{equation}

The following lemma provides a bound on the second moment:
\begin{lemma} \label{lem:comb_var}
 The second moment of the estimator defined in Equation~\eqref{estcomb} satisfies
\begin{align*}
 \sum_{\left(i^{(1)},\ldots,i^{(K)}\right) \in \A_r} p_{\left(i^{(1)},\ldots,i^{(K)}\right)}(t) 
&\E_t\left[\wt c\left(i^{(1)},\ldots,i^{(K)}, r, d_t^{(1)},\ldots,d_t^{(K)}\right)^2 \right] 
\\
&\le 16 K^2 \beta^2 
\log\pa{\frac{KN^3}{\gamma} + N + 2} + 16 K^2\beta^2 + 2(f+K\beta)^2.
\end{align*}
\end{lemma}

\begin{proof}
See Appendix 1.
\end{proof}

The following theorem establishes a performance guarantee of the proposed allocation policy.

\medskip\begin{theorem} \label{ewfthmcomb}
Consider the ``combinatorial'' version of the repeated newsvendor problem described above. The expected regret of the EWF policy with the cost estimator in Eq.\ \eqref{estcomb} and parameters
\[
\gamma=\frac{1}{T}, \qquad \eta = \sqrt{\frac{\log N}{\beta^2 KT \log \left(TNK + N + 2 \right)}},
\]
is bounded from above as 
\[
 \E[\R(T)] = O\pa{K^{3/2}\beta\sqrt{T \log N \log\pa{TNK}}}.
\]
\end{theorem}

\begin{proof}
See Appendix 1.
\end{proof}

As in the single-warehouse setting, the scaling of the expected regret of the EWF policy is near-optimal with respect to both $T$ and $N$, 
similarly to the single-retailer setting in Section 2. Again, we highlight the fact that the logarithmic scaling with respect to $N$ is made 
possible by our carefully designed estimator that assigns non-zero entries to all inventory levels below the realized inventory decision.
On the other hand, the $O\left(K^{3/2}\right)$ scaling with respect to the number of retailers is 
unlikely to be optimal. 
To see this, we note that the combinatorial version of the repeated newsvendor problem falls into the broader 
class of online combinatorial optimization, for which the limitations of the EWF policy are now clearly understood. In particular, 
\citet{ABL14} shows that the EWF policy has, in general, suboptimal performance guarantees in terms of the size of the 
decision set, essentially translating to a suboptimal scaling with respect to $K$ in our case. We conjecture that the optimal scaling with
respect to $K$ is linear, and it can be achieved by a suitable adaptation of the Component Hedge algorithm in \citet{KWK10} (also called 
Online Stochastic Mirror Descent in \citet{ABL14}). We omit a detailed treatment of that direction in order to maintain the clarity of our 
presentation. We also remark that a similar combination of feedback graphs and combinatorial decision sets is studied in \citet{KNVM14}. 
An adaptation of their algorithm, termed FPL-IX, combined with our cost estimates, can be shown to satisfy a regret bound identical to our Theorem~\ref{ewfthmcomb}. 
Details are again omitted for brevity.

Note that from a mathematical standpoint, the role of the cost that the warehouse incurs at every time period is to couple the inventory 
management problems of the different retailers. However, the exact functional form of $c_0\left(I_t^{(1)},\ldots,I_t^{(K)}\right)$ is never 
used. Thus, the approach presented above extends in a straightforward way to the case where the supply chain operates in a ``pull'' manner, 
from downstream to upstream. In that case, the inventory manager has no incentive to keep any inventory at the warehouse, since the product 
is perishable. The only cost that the warehouse incurs is a fixed ordering cost of $f^{(0)}$ for procuring the required inventory from 
upstream:
\[
c_0\left(I_t^{(1)},\ldots,I_t^{(K)}\right) = f^{(0)} \II{\sum_{k \in \K} I_t^{(k)}>0}.
\]
With this minor modification, the rest of the analysis follows verbatim.


\bigskip\section{Concluding Remarks} \label{conc}

We conclude the paper by drawing connections between our setting and two related problems
which have attracted the attention of academic research recently,
and discussing some directions for future research.

\subsection{Two Related Problems}

The first problem is bidding in repeated second-price auctions with valuation learning: a bidder participates in second-price auctions for 
different products, with her valuations of these products being unknown a priori. The bidder learns her true valuation of a given product 
only if she wins the respective auction, that is, if she submits the highest bid, and her reward in that case is equal to the difference 
between that valuation and the second highest bid; otherwise, the bidder learns nothing about her valuation and earns no reward. The goal of 
the bidder is to maximize her expected reward, which translates naturally into an exploration-exploitation tradeoff in her bidding strategy. 
\cite{WPR16} studies this problem from a nonstochastic viewpoint, and proposes a bidding policy whose expected regret scales optimally with 
respect to the number of auctions. There is an intriguing similarity to our problem, in that the learner receives feedback only when her 
action is greater than the opponent's, in which case the problem reduces to a full-information one. The cost structure, however, is quite 
different: in a newsvendor setting the amount of cost incurred depends on the inventory decision of the manager, whereas in a second-price 
auction the size of the reward is independent of the bid, conditional on winning/not winning the auction. The authors take advantage of this 
property by introducing a variation of the EWF policy that is based on interval splitting. Their approach does not seem applicable to our 
setting, and our cost estimator is, not surprisingly, quite different.

The second problem is sequential stock allocation to dark pools: a trader receives amounts of stocks to liquidate at different periods, and 
allocates these amounts among different dark pools, whose demand for the particular stocks is unknown a priori. The trader learns the demand 
at a certain dark pool only if the amount of stocks allocated there exceeds it. This problem is formulated as a dynamic learning problem 
with limited feedback in \cite{GNKV10}, and analyzed from a nonstochastic viewpoint with regret guarantees in \cite{ABD10}. It is not hard 
to see that this is almost identical to the ``combinatorial'' newsvendor problem in Section \ref{combnews}, with the main difference being 
that the amount of inventory that the warehouse receives is constant, whereas the amount of stocks that the trader tries to liquidate may 
vary. (Relatedly, the notion of regret that the authors adopt is somewhat different than ours.) Hence, one can argue that we analyze a 
special case of the problem in \cite{ABD10}. On the other hand, the expected regret of the policy proposed in the latter paper, for the case 
of integral allocations, scales as $O\left(T^{2/3}\right)$, a result which we improve considerably upon by taking advantage of the local 
observability property.

\subsection{Open Questions}

A limitation of our results, as well as of most results on partial monitoring problems, is that they only hold in 
expectation, and assuming ``oblivious'' adversaries. While the latter restriction can be easily relaxed by standard techniques
(see the discussion in \citealp{ADT12}), extending the results to hold with high probability seems like a significant challenge. Indeed, 
all known algorithms that have such guarantees, such as the Exp3.P policy of \citet{ACFS02} or the Exp3-IX policy of \citet{Neu15b},
use optimistically biased estimates of the costs, whereas our proofs crucially rely on using unbiased estimators of cost differences 
for arbitrary pairs of actions. It is unclear how one can construct estimators that would provide a suitable optimistic bias, simultaneously 
for all pairs of actions.

Regarding the single-warehouse multi-retailer variant of our setting, we note that the policy proposed in Section~\ref{combnews} is not 
computationally efficient in its current form. However, it is already unclear how to compute an optimal solution in an efficient way even 
in the offline variant of this problem (i.e., computing an optimal allocation in perfect knowledge of the demands). If such an 
efficient way exists, a natural extension of the FPL-IX policy in \citet{KNVM14} would also admit a computationally efficient implementation, 
thus settling this question. On the other hand, if the offline problem turns out to be hard, it is unreasonable to expect that there is an 
efficient algorithm for the online variant that we consider. We leave the study of this offline optimization problem as future work.

In Section~\ref{TrackReg}, we provide an extension to the EWF that guarantees bounds on the tracking regret. While this performance 
measure is already much more expressive than regret against the best fixed action, we note that there are several, even stronger baselines 
that can be considered. Two particular performance notions are the strongly adaptive regret, which compares the performance of 
the learner to the best action within every subinterval of $[1,T]$, see \cite{Z03,HS09,DGS15,AKCV16}; and the dynamic regret, which uses the 
best sequence of actions as a comparator, see \cite{HW13,BGZ14,BGZ15,JRSS15,KO16}. The FSF scheme, proposed in 
Section~\ref{TrackReg}, can be already shown to guarantee non-trivial bounds on its adaptive regret, as revealed by an inspection of the 
proof of Theorem~\ref{fsfthm}, yet it is unclear whether non-trivial bounds can be achieved on 
the strongly adaptive regret or the dynamic regret in our setting. Regarding the dynamic regret, we find it likely that the techniques of 
\citet{BGZ15} can be adapted to our setting to achieve meaningful bounds, although one may need additional tools from \citet{BGZ14} and 
\citet{KO16} to deal with partial feedback. The case of strong adaptivity seems to be significantly more complicated: as shown 
by \citet{DGS15}, it is impossible to achieve non-trivial strongly adaptive regret guarantees in the multi-armed bandit setting. Since, feedback-wise, our 
setting is situated between the multi-armed bandit and full information, this negative result does not rule out the possibility to devise 
strongly adaptive algorithms. We leave this investigation, again, for future work.


\section*{Acknowledgements}

G.\ Lugosi was supported by the Spanish Ministry of Economy, Industry and Competitiveness through the grant Grant MTM2015-67304-P 
(AEI/FEDER, UE). M.G.\ Markakis was supported by the Spanish Ministry of Economy, Industry and Competitiveness through a Juan de la Cierva 
fellowship, and the grant ECO2016-75905-R (AEI/FEDER, UE). G.\ Neu was supported by the UPFellows Fellowship (Marie Curie COFUND program 
n${^\circ}$ 600387).

\bibliographystyle{ormsv080}

\bibliography{bibl}

\section*{Appendix 1: Proofs of Main Results}

\subsection*{Proof of Theorem \ref{ewfthm}}

Our analysis largely follows the steps of the proof of Theorem~3.1 by \citet{ACFS02}, combined with our Lemmas~\ref{estprop} 
and~\ref{unbias}. We analyze a slightly more general version of the EWF algorithm that uses an arbitrary exploration distribution 
$\mu$, with $\mu_i$ being the probability of taking action $i$ in exploration rounds. More precisely, we  consider a version of EWF 
that computes its probability distributions over the actions as 
\[
 p_i(t) = (1-\gamma) \frac{W_i(t-1)}{W(t-1)} + \gamma \mu_i, \qquad i \in \I.
\]
The statement will follow from setting $\mu_i = 1/N$ for all actions $i$.

The key idea of the analysis is studying the term $\log\big(W(T)/W(0)\big)$ which, as we show shortly, relates 
closely to the regret. We start by constructing a lower bound:
\begin{align}
\log \left(\frac{W(T)}{W(0)}\right) &= \log \big(W(T)\big) - \log \big(W(0)\big) \nonumber \\
&= \log \left( \sum_{i \in \I} W_i(T) \right) - \log \left(\sum_{i \in \I} W_i(0) \right) \nonumber \\
&= \log \left( \sum_{i \in \I} e^{-\eta \wt C_i(T)} \right) - \log N \nonumber \\
&\geq \log \left(e^{-\eta \wt C_{i^*}(T)} \right) - \log N \nonumber \\
&= - \eta \sum_{t \in \T} \wt c(i^*,d_t) - \log N, \label{lb}
\end{align}
where $i^*$ is the best fixed action in hindsight, for the particular demand sequence.

Then, we derive an upper bound on $\log\big(W(T)/W(0)\big)$:
\begin{align*}
\log \left(\frac{W(T)}{W(0)}\right) &=  \log \left(\prod_{t \in \T} \frac{W(t)}{W(t-1)}\right) \\
&=  \sum_{t \in \T} \log \left(\frac{W(t)}{W(t-1)}\right) \\
&=  \sum_{t \in \T} \log \left(\sum_{i \in \I} \frac{W_i(t)}{W(t-1)}\right) \\
&= \sum_{t \in \T} \log \left(\sum_{i \in \I} \frac{W_i(t-1)}{W(t-1)} e^{-\eta \wt c(i,d_t)}\right)
\quad \text{(by Eq. \eqref{dyn}).}
\end{align*}
Note that $e^{-x} \leq 1- x + x^2/2$, for all $x \geq 0$, and our estimators for the cost of the different decisions are nonnegative. Thus,
\begin{align*}
\log \left(\frac{W(T)}{W(0)}\right) &\leq \sum_{t \in \T} \log \left(\sum_{i \in \I} \frac{W_i(t-1)}{W(t-1)} \left(1-\eta \wt c(i,d_t) + \frac{\eta^2}{2} \wt c(i,d_t)^2\right)\right) \\
&= \sum_{t \in \T} \log \left(1-\eta \sum_{i \in \I} \frac{W_i(t-1)}{W(t-1)} \wt c(i,d_t) + \frac{\eta^2}{2} \sum_{i \in \I} \frac{W_i(t-1)}{W(t-1)} \wt c(i,d_t)^2\right).
\end{align*}
Moreover, $\log(1+x) \leq x$, for all $x>-1$. Since this is the case with the right-hand side of the expression above (being an upper bound 
to a sum of exponential terms), we have, using Eq. \eqref{ewf}, that
\begin{align}
\log \left(\frac{W(T)}{W(0)}\right) &\leq \sum_{t \in \T} \left(-\eta \sum_{i \in \I} \frac{W_i(t-1)}{W(t-1)} \wt c(i,d_t) + 
\frac{\eta^2}{2} \sum_{i \in \I} \frac{W_i(t-1)}{W(t-1)} \wt c(i,d_t)^2\right) \nonumber \\
&= \sum_{t \in \T} \left(-\eta \sum_{i \in \I} \frac{p_i(t)-\gamma\mu_i}{1-\gamma} \wt c(i,d_t) + 
\frac{\eta^2}{2} 
\sum_{i \in \I} \frac{p_i(t)-\gamma\mu_i}{1-\gamma} \wt c(i,d_t)^2\right). \label{ub}
\end{align}

Eqs. \eqref{lb} and \eqref{ub} imply that
\[
\frac{\eta}{1-\gamma} \sum_{t \in \T} \sum_{i \in \I} p_i(t) \wt c(i,d_t) - \eta \sum_{t \in \T} \wt c(i^*,d_t) \leq \log N + 
\frac{\eta\gamma}{1-\gamma} \sum_{t \in \T} \sum_{i \in \I} \mu_i \wt c(i,d_t) + \frac{\eta^2}{2(1-\gamma)}\sum_{t \in \T} 
\sum_{i \in \I} p_i(t) \wt c(i,d_t)^2.
\]
Since $\wt c(i^*,d_t) \geq 0$, for all $t \in \T$, we have that
\[
\sum_{t \in \T} \sum_{i \in \I} p_i(t) \wt c(i,d_t) - \sum_{t \in \T} \wt c(i^*,d_t) \leq \frac{\log N}{\eta} + \gamma \sum_{t \in 
\T} \sum_{i \in \I} \mu_i \wt c(i,d_t) + \frac{\eta}{2}\sum_{t \in \T} \sum_{i \in \I} p_i(t) \wt c(i,d_t)^2.
\]
This further implies that
\[
\EE{\sum_{t \in \T} \sum_{i \in \I} p_i(t) \big(\wt c(i,d_t) - \wt c(i^*,d_t)\big)} \leq \frac{\log N}{\eta} + \gamma 
\E\left[\sum_{t \in \T} \sum_{i \in \I} \mu_i \wt c(i,d_t) \right] + \frac{\eta}{2}\E\left[\sum_{t \in \T} \sum_{i \in \I} p_i(t) \wt 
c(i,d_t)^2 \right].
\]
The tower rule of expectations along with Eq.~\eqref{expreg} and Lemma \ref{unbias} imply that
\begin{align*}
 \EE{\sum_{t \in \T} \sum_{i \in \I} p_i(t) \big(\wt c(i,d_t) - \wt c(i^*,d_t)\big)} &= \EE{\sum_{t \in \T} \sum_{i \in \I} p_i(t) \E_t 
\big[\wt c(i,d_t) - \wt c(i^*,d_t) \big] } 
\\
&= \EE{\sum_{t \in \T} \sum_{i \in \I} p_i(t) \big(c(i,d_t) - c(i^*,d_t)\big)} 
= \E[\R(T)].
\end{align*}
Combined with Lemma \ref{estprop}, and the fact that $\mu_i = 1/ N$, we get:
\begin{align*}
\E[\R(T)] &\leq \frac{\log N}{\eta} + \frac{\gamma}{N} \sum_{t \in \T} \sum_{i \in \I} 2 \beta + \frac{\eta}{2}\sum_{t \in \T} \sum_{i \in 
\I} p_i(t) \frac{4 \beta^2}{\P_t(I_t \geq i)} \\
&= \frac{\log N}{\eta} + 2 \beta \gamma T + 2 \beta^2 \eta \sum_{t \in \T} \sum_{i \in \I} \frac{p_i(t)}{\P_t(I_t \geq i)}.
\end{align*}

The final step in the proof is to bound the term $\sum_{i \in \I} p_i(t) \big/ \P_t(I_t \geq i)$. We have that
\[
\sum_{i \in \I} \frac{p_i(t)}{\P_t(I_t \geq i)} = \sum_{i \in \I} \frac{p_i(t)}{p_i(t) + \sum_{j \in \I: j > i} p_j(t)} \leq 2 \log \left(\frac{N^3}{\gamma} + N + 2 \right) + 2.
\]
This follows directly from Lemma 16 in \cite{ACGMMSpre} with the following correspondences: the number of nodes in the ``feedback graph'' 
is equal to $N$ in our setting; the independence number of the feedback graph is equal to 1 in our setting; a lower bound on the 
probabilities of choosing the different actions is $\gamma/N$ in our setting; and a dominating set of the feedback graph in our case is the 
single-element set that contains the largest element of $\I$.

Consequently,
\[
\E[\R(T)] \leq \frac{\log N}{\eta} + 2 \beta \gamma T + 4 \beta^2 \eta T \log \left(\frac{N^3}{\gamma} + N + 2 \right) + 4 \beta^2 \eta T.
\]
By choosing
\[
\eta = \sqrt{\frac{\log N}{4 \beta^2 T \log \left(\frac{N^3}{\gamma} + N + 2 \right)}},
\]
we have that
\[
\E[\R(T)] \leq 4 \beta \sqrt{T \log N \log \left(\frac{N^3}{\gamma} + N + 2 \right)} + 2 \beta \sqrt{T \log N} + 2 \beta \gamma T.
\]
Finally, by setting $\gamma = 1 / 2 \beta T$, we get:
\[
\E[\R(T)] \leq 4 \beta \sqrt{T \log N \log \left(2 \beta T N^3 + N + 2 \right)} + 2 \beta \sqrt{T \log N} + 1.
\]

\bigskip\subsection*{Proof of Theorem \ref{fsfthm}}

We follow the steps of the proof of Theorem~8.1 in \citet{ACFS02}. To this end, fix a comparator sequence $i_{[T]} \in \I^T_S$, 
and partition the interval $[1,T]$ into a number of subintervals $I_1 = [1,T_1]$, $I_2 = [T_1 + 1,T_2]$, 
$\dots$, $I_C = [T_{C-1}+1,T]$, such that $i_t$ remains constant within each interval. Since $i_{[T]}\in \I^T_S$, we have that
$C\le S$. In the remainder of the proof, we bound the regret within each interval, and then combine the obtained bounds to prove a 
guarantee about the (global) tracking regret. 

Fix an arbitrary interval $I_s$, $s \in \{1,\ldots,C\}$, and let $j_s$ be the action taken during that 
interval (i.e., $i_t = j_s$, for all $t\in I_s$). Also, let $\tau_s = T_{s} - T_{s-1}$ be the length of $I_s$. As in the proof of 
Theorem~\ref{ewfthm}, we study the term $\log\big(W(T_s)/W(T_{s-1})\big)$. First, note that 
\begin{align*}
 W_{j_s}(T_s) &\ge W_{j_s}(T_{s-1}+1) \exp\pa{-\eta \sum_{t=T_{s-1}+2}^{T_s} \tc(j_s,d_t)}
 \\
 &\ge \frac{\alpha}{N} W(T_{s-1}) \exp\pa{-\eta \sum_{t=T_{s-1}+2}^{T_s} \tc(j_s,d_t)}
 \\
 &\ge \frac{\alpha}{N} W(T_{s-1}) \exp\pa{-\eta \sum_{t=T_{s-1}+1}^{T_s} \tc(j_s,d_t)},
\end{align*}
where the first and second inequalities follow from the (recursive) definition of the sequence of weights $W_j(t)$, and the last one 
is a consequence of the non-negativity of the loss estimates $\tc(i,d_t)$. Hence,
\[
 \log\pa{\frac{W(T_s)}{W(T_{s-1})}} \ge \log\pa{\frac{W_{j_s}(T_s)}{W(T_{s-1})}} \ge \log \pa{\frac{\alpha}{N}} - \eta 
\sum_{t=T_{s-1}+1}^{T_s} \tc(j_s,d_t).
\]
On the other hand, we have that
\[
 \frac{W(t+1)}{W(t)} = \sum_{i\in\I} \frac{W_i(t)e^{-\eta \tc(i,d_t)} + \frac{\alpha}{N} W(t)}{W(t)} = \sum_{i\in\I} \frac{W_i(t)e^{-\eta 
\tc(i,d_t)}}{W(t)} + \alpha,
\]
so by a similar line of reasoning as in the proof of Theorem~\ref{ewfthm}, it can be verified that
\[
 \log \frac{W(T_s)}{W(T_{s-1})} \le \sum_{t = T_{s-1}+1}^{T_{s}} \left(-\eta \sum_{i \in \I} \frac{p_i(t)-\gamma/N}{1-\gamma} \wt 
c(i,d_t) + \frac{\eta^2}{2} \sum_{i \in \I} \frac{p_i(t)-\gamma/N}{1-\gamma} \wt c(i,d_t)^2 + \alpha\right).
\]

Combining the two bounds together, taking expectations, and appealing to Lemma~\ref{estprop}, we obtain:
\begin{equation}\label{eq:fsf_almost}
\EE{\sum_{t=T_{s-1}+1}^{T_s} \sum_{i \in \I} p_i(t) \big(\wt c(i,d_t) - \wt c(j_s,d_t)\big)} \leq \frac{\log 
\pa{\frac{N}{\alpha}}}{\eta} + 2\beta\gamma \tau_s + 2 \beta^2 \eta \EE{\sum_{t = T_{s-1}+1}^{T_s} \sum_{i \in \I} \frac{p_i(t)}{\P_t(I_t 
\geq i)}} + \alpha \tau_s.
\end{equation}
As in the proof of Theorem~\ref{ewfthm}, the third term on the right-hand side can be bounded from above as
\begin{align*}
\sum_{i \in \I} \frac{p_i(t)}{\P_t (I_t \geq i)} \leq 2 \log \left(\frac{N^3}{\gamma} + N + 2 \right) + 2.
\end{align*}
Using this bound, we can add over all intervals $s \in \{1,\ldots,C\}$ both sides of Eq.~\eqref{eq:fsf_almost}, and use Lemma~\ref{unbias} 
to obtain:
\[
 \EE{\sum_{t=1}^T \big(c(I_t,d_t) - c(i_t,d_t)\big)} \le \frac{S \log \pa{\frac{N}{\alpha}}}{\eta} + 
2\beta\gamma T + 4 \beta^2 \eta T \log \left(\frac{N^3}{\gamma} + N + 2 \right) + \alpha T.
\]
The statement of the theorem follows from taking the supremum over all $i_{[T]}\in \I^T_S$\footnote{We note that supremum and expectation 
can be interchanged in our case, since the comparator sequence, that the supremum is taken with respect to, is deterministic. 
This would not have been the case, e.g., if the firm competed against an adaptive adversary.}, and substituting for the chosen values of $\gamma$, 
$\eta$, and $\alpha$.

\bigskip\subsection*{Proof of Lemma \ref{lem:comb_var}}

For simplicity, let us introduce the notation
\[
 \ell_k\left(i^{(k)},d_t^{(k)}\right) = v_i^{T} \S_i e_{d_t^{(k)}} +f^{(k)} \II{i^{(k)}>0} + \beta,
\]
so that each retailer's cost estimate can be written as
\[
\wt c_k\left(i^{(k)},d_t^{(k)}\right) = \frac{\II{I_t^{(k)} \geq i^{(k)}} \ell_k\left(i^{(k)},d_t^{(k)}\right)}{\P_t \left(I_t^{(k)}
\geq i^{(k)}\right)},
\]
for all $k\in\K$. Also, let $\tI_t$ be an independent copy of $I_t$. With this notation, we have that
\begin{align*}
 &\sum_{\left(i^{(1)},\ldots,i^{(K)}\right) \in \A_r} p_{\left(i^{(1)},\ldots,i^{(K)}\right)}(t) 
\E_t\left[\wt c\left(i^{(1)},\ldots,i^{(K)}, r, d_t^{(1)},\ldots,d_t^{(K)}\right)^2 \right] \\
= &\E_t\left[\pa{c_0\left(\tI_t^{(1)},\ldots,\tI_t^{(K)}, r\right) + \sum_{k\in\K} \wt c_k\left(\tI_t^{(k)},d_t^{(k)}\right)}^2 \right] \\
\le &2\E_t\left[c_0\left(\tI_t^{(1)},\ldots,\tI_t^{(K)}, r\right)^2 + \pa{\sum_{k\in\K} \wt c_k\left(\tI_t^{(k)},d_t^{(k)}\right)}^2 \right],
\end{align*}
where the last step follows from the inequality $(a+b)^2 \le 2\pa{a^2 + b^2}$, which holds for all $a,b\in\real$. The first term can be 
trivially bounded by $(f+K\beta)^2$. Regarding the second term, we have that
\begin{align*}
&\E_t\left[\pa{\sum_{k\in\K} \wt c_k\left(\tI_t^{(k)},d_t^{(k)}\right)}^2 \right] \\
= &\E_t\left[\pa{\sum_{j\in\K}  \frac{\II{I_t^{(j)} \geq \tI_t^{(j)}}}{\P_t \left(I_t^{(j)} \geq \tI_t^{(j)}\right)} 
\ell_j\left(\tI_t^{(j)},d_t^{(j)}\right)}\cdot \pa{\sum_{k\in\K}  \frac{\II{I_t^{(k)} \geq \tI_t^{(k)}}}{\P_t \left(I_t^{(k)} \geq \tI_t^{(k)}\right)} 
\ell_k\left(\tI_t^{(k)},d_t^{(k)}\right)} \right] \\
= &\E_t\left[\sum_{j\in\K} \sum_{k\in\K} \frac{\II{I_t^{(j)} \geq \tI_t^{(j)}} \II{I_t^{(k)} \geq 
\tI_t^{(k)}}}{\P_t \left(I_t^{(j)}\geq \tI_t^{(j)}\right) \P_t \left(I_t^{(k)}\geq \tI_t^{(k)}\right)} 
\ell_j\left(\tI_t^{(j)},d_t^{(j)}\right) \ell_k\left(\tI_t^{(k)},d_t^{(k)}\right) \right] \\
\le &\frac{1}{2} \E_t\left[\sum_{j\in\K} \sum_{k\in\K} \pa{\frac{1}{\P_t \left(I_t^{(j)}\geq \tI_t^{(j)}\right)^2} + 
\frac{1}{\P_t \left(I_t^{(k)}\geq \tI_t^{(k)}\right)^2}}\II{I_t^{(j)} \geq \tI_t^{(j)}} \II{I_t^{(k)} \geq \tI_t^{(k)}} 
\ell_j\left(\tI_t^{(j)},d_t^{(j)}\right) \ell_k\left(\tI_t^{(k)},d_t^{(k)}\right) \right],
\end{align*}
again using $(a+b)^2 \le 2\pa{a^2 + b^2}$. We further have that
\begin{align*}
&\E_t\left[\pa{\sum_{k\in\K} \wt c_k\left(\tI_t^{(k)},d_t^{(k)}\right)}^2 \right] \\
= &\E_t\left[\sum_{j\in\K} \sum_{k\in\K} \frac{1}{\P_t \left(I_t^{(j)}\geq \tI_t^{(j)}\right)^2} \II{I_t^{(j)} \geq \tI_t^{(j)}} \II{I_t^{(k)} 
\geq \tI_t^{(k)}} \ell_j\left(\tI_t^{(j)},d_t^{(j)}\right) \ell_k\left(\tI_t^{(k)},d_t^{(k)}\right) \right],
\end{align*}
due to the symmetry between $j$ and $k$, which implies that
\begin{align*}
&\E_t\left[\pa{\sum_{k\in\K} \wt c_k\left(\tI_t^{(k)},d_t^{(k)}\right)}^2 \right] \\
= &\E_t\left[\sum_{j\in\K} \frac{1}{\P_t \left(I_t^{(j)}\geq \tI_t^{(j)}\right)^2} \II{I_t^{(j)} \geq \tI_t^{(j)}} 
\ell_j\left(\tI_t^{(j)},d_t^{(j)}\right) \sum_{k\in\K} \II{I_t^{(k)} \geq \tI_t^{(k)}} \ell_k\left(\tI_t^{(k)},d_t^{(k)}\right) \right] \\
\le &2K\beta\E_t\left[\sum_{j\in\K} \frac{1}{\P_t \left(I_t^{(j)}\geq \tI_t^{(j)}\right)^2} \II{I_t^{(j)} \geq \tI_t^{(j)}} \ell_j\left(\tI_t^{(j)},d_t^{(j)}\right)\right] \\
= &2K\beta\E_t\left[\sum_{j\in\K} \frac{\ell_j\left(\tI_t^{(j)},d_t^{(j)}\right)}{\P_t \left(I_t^{(j)}\geq \tI_t^{(j)}\right)} \right] \\
\le &4K\beta^2 \sum_{j=1}^K \E_t\left[\sum_{i=1}^N \frac{\P_t\pa{I_t^{(j)} = i}}{\P_t \left(I_t^{(j)}\geq i\right)} \right],
\end{align*}
where the inequalities follow from bounding from above $\ell_j (\cdot,\cdot)$ by $2\beta$. 

It remains to bound the sums within the expectation, for all $j$. To this end, we observe that our exploration distribution $\mu$ guarantees that $\PP{I_t^{(j)}=i}\ge \frac{\gamma}{NK}$ holds for all $i,j$. Given this lower bound, we can apply Lemma~16 of \citet{ACGMMSpre} as done in the proof of Theorem~\ref{ewfthm}:
\[
\sum_{i=1}^N \frac{\P_t\pa{I_t^{(j)} = i}}{\P_t \left(I_t^{(j)}\geq i\right)} \le 2 \log \left(\frac{KN^3}{\gamma} + N + 2 \right) + 2.
\]
Putting everything together, we obtain the desired result.

\bigskip\subsection*{Proof of Theorem \ref{ewfthmcomb}}

By following closely the proof of Theorem \ref{ewfthm} with our definition of $\mu$ and applying Eq.\ \eqref{locobscomb}, we have that
\begin{align*}
\E[\R_c(T)] \leq &\frac{\log |\A_r|}{\eta} + \gamma \sum_{t \in \T} \sum_{\left(i^{(1)},\ldots,i^{(K)}\right) \in \A_r} 
\mu_{\left(i^{(1)},\ldots,i^{(K)}\right)}
\E_t\left[\wt c\left(i^{(1)},\ldots,i^{(K)}, r, d_t^{(1)},\ldots,d_t^{(K)}\right) \right] \\
&+ \frac{\eta}{2}\sum_{t \in \T} \sum_{\left(i^{(1)},\ldots,i^{(K)}\right) \in \A_r} p_{\left(i^{(1)},\ldots,i^{(K)}\right)}(t) 
\E_t\left[\wt c\left(i^{(1)},\ldots,i^{(K)}, r, d_t^{(1)},\ldots,d_t^{(K)}\right)^2 \right] \\
\leq &\frac{K \log N}{\eta} + \gamma \big(f + K(2 \beta + f)\big) T \\
&+ \eta \pa{8 K^2 \beta^2 \log\pa{\frac{KN^3}{\gamma} + N + 2} + 8 K^2\beta^2 + (f+hr)^2},
\end{align*}
where the last step uses Eq.~\eqref{meancomb} and Lemma~\ref{lem:comb_var}. Substituting the prescribed values for $\gamma$ and $\eta$ 
yields the statement of the theorem.


\section*{Appendix 2: The AIM Policy}

Let us briefly describe the AIM policy, introduced in \citet{HR09}, originally proposed for i.i.d.~demand sequences. In the case of discrete
demand--see the AIM-discrete variant, in Section~3.4 of \citet{HR09}--the policy has a guaranteed performance only in the ``partially censored'' 
case, that is, if the learner has access to a lost-sales indicator. It is based on the online gradient descent method of \citet{Z03}: 
the key idea is to compute recursively the sequence of auxiliary points
\[
 x_{t+1} = \Pi\pa{x_t - \alpha_t g_t},
\]
where $g_t$ is an estimate of the left derivative\footnote{In general, the role of the left derivative can be played by any element of the 
subdifferential; we avoid this terminology to maintain clarity of exposition.} of the loss function $c(\cdot,d_t)$, evaluated at $x_t$, 
and $\Pi$ is a projection operator on the interval $[0,D]$. Recall that $c(x,d_t)$, as defined in Equation~\eqref{newscost}, is convex, so it 
has a well-defined left derivative: $-h$ for $x < d_t$, $b$ for $x>d_t$, and $[-h,b]$ for $x=d_t$. With the help of this sequence of 
auxiliary points, the actual (discrete) inventory levels are chosen randomly as follows:
\[
 I_t = \left\lfloor x_t\right\rfloor + B_t,
\]
where $B_t$ is a Bernoulli random variable with expectation $x_t - \left\lfloor x_t\right\rfloor$, so that $\EEt{I_t} = x_t$. The most pertinent
question concerns the estimation of the gradients of $c(\cdot,d_t)$ in a reliable way. \citet{HR09} suggest to use
\[
 g_t = \begin{cases}
        - b + \pa{h+b}\II{d_t\le I_t}, &\mbox{if $I_t = \left\lfloor x_t\right\rfloor$}
        \\
        - b + \pa{h+b}\II{d_t\le I_t - 1}, &\mbox{if $I_t = \left\lceil x_t\right\rceil$,}
       \end{cases}
\]
which is an unbiased estimator of the left derivative. Crucially, however, the construction of this operator requires access to the 
lost sales indicator $\II{d_t\le I_t}$, which on many occasions may not be available. 
In the absence of such an indicator, one can think of using the simpler estimator
\[
 \wh{g}_t = - b + \pa{h+b}\II{d_t\le I_t - 1},
\]
instead of $g_t$ in the policy described above; this is actually the scheme that we implement in the numerical experiments of Section~\ref{Sim}.
While this estimator can be computed given censored/sales data, it is easy to see that it is a \textit{biased} estimator of the left derivative. 
In particular, since $\II{d_t \le I_t -1} > \II{d_t \le I_t}$, $\wh{g}_t$ consistently overestimates the true derivative. This shortcoming severely 
impacts the behavior of the policy: even for constant demand realizations, it always overshoots the optimal inventory 
level. This effect is illustrated in Figure~\ref{fig:aim}, where the AIM policy has a linearly increasing regret in a very simple scenario: 
three possible order levels, $\ev{0,1,2}$, and a constant demand of $1$.

\begin{figure}
\begin{center}
\includegraphics[width=.5\textwidth]{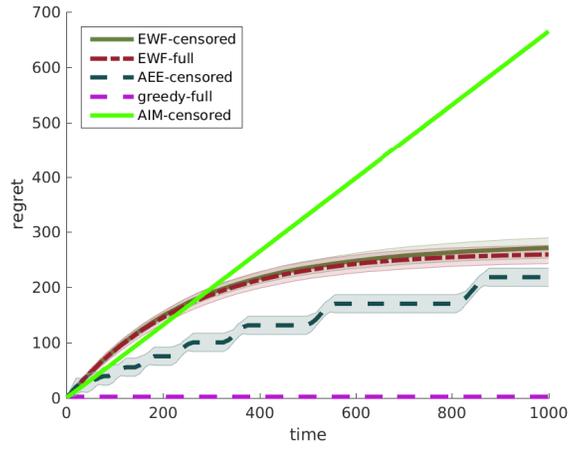}
\caption{The misbehavior of the AIM policy.}\label{fig:aim}
\end{center}
\end{figure}

\end{document}